\documentclass{article}

\PassOptionsToPackage{numbers}{natbib}
\usepackage[preprint]{neurips_2025}

\newcommand\nonumfootnote[1]{%
\begingroup%
    \renewcommand\thefootnote{}\footnote{\hspace{-4pt}#1}%
    \addtocounter{footnote}{-1}%
\endgroup%
}

\usepackage[utf8]{inputenc} 
\usepackage[T1]{fontenc}    
\usepackage{hyperref}       
\usepackage{url}            
\usepackage{booktabs}       
\usepackage{amsfonts}       
\usepackage{nicefrac}       
\usepackage{microtype}      
\usepackage{xcolor}         
\usepackage{amsmath}
\allowdisplaybreaks[4]
\usepackage{amssymb}
\usepackage{mathtools}
\usepackage{amsthm}
\usepackage{thmtools}
\usepackage{thm-restate}
\usepackage{xspace}
\usepackage{overpic}
\usepackage[capitalize,noabbrev]{cleveref}
\usepackage{enumitem}
\usepackage{subcaption}
\usepackage{algorithm}
\usepackage{algpseudocode}
\definecolor{mydarkblue}{rgb}{0,0.08,0.45}
\hypersetup{
    colorlinks=true,
    linkcolor=mydarkblue,
    citecolor=mydarkblue,
    filecolor=mydarkblue,
    urlcolor=mydarkblue,
}

\usepackage[toc, page, header]{appendix}
\newcommand{\tableofappendix}{\renewcommand{\contentsname}{Contents of Appendix}\newpage\tableofcontents\addtocontents{toc}{\protect\setcounter{tocdepth}{3}}\clearpage}

\theoremstyle{plain}
\newtheorem{theorem}{Theorem}[section]

\newtheorem{lemma}[theorem]{Lemma}

\theoremstyle{definition}
\newtheorem{definition}{Definition}[section]

\newtheorem{assumption}{Assumption}[section]
\crefname{assumption}{Assumption}{Assumptions}

\theoremstyle{remark}

\def\phenomenon{{instability\xspace}}
\def\Phenomenon{{Instability\xspace}}
\def\PHenomenon{{Instability\xspace}}

\def\intrinsiceffect{{intrinsic instability\xspace}}
\def\Intrinsiceffect{{Intrinsic instability\xspace}}

\def\intrinsiceffectmetric{{intrinsic instability coefficient\xspace}}
\def\Intrinsiceffectmetric{{Intrinsic instability coefficient\xspace}}

\def\effect{{instability\xspace}}
\def\Effect{{Instability\xspace}}

\def\effectmetric{{instability coefficient\xspace}}
\def\effectmetrics{{instability coefficients\xspace}}
\def\Effectmetric{{Instability coefficient\xspace}}
\def\Effectmetric{{Instability coefficients\xspace}}

\newcommand{\myparagraph}[1]{{\noindent\textbf{#1}\xspace}}

\definecolor{figblue}{HTML}{1F77B4}
\definecolor{figgreen}{HTML}{2CA02C}
\definecolor{figred}{HTML}{D62728}

\newcommand{\todo}[1]{{\color{red}#1}}
\newcommand{\TODO}[1]{\textbf{\color{red}[TODO: #1]}}
\newcommand{\TOCITE}[1]{\textbf{\color{blue}[TOCITE: #1]}}
\newcommand{\LATER}[1]{\textbf{\color{green}[LATER: #1]}}
\renewcommand{\TODO}[1]{}
\renewcommand{\todo}[1]{}
\renewcommand{\TOCITE}[1]{}
\renewcommand{\LATER}[1]{}

\usepackage{amsmath,amsfonts,bm}

\def\dd{{\rm d}}
\def\supp{{\rm supp}}

\def\defeq{{:=}}
\def\min{{\mathrm{min}}}
\def\max{{\mathrm{max}}}

\def\datadist{{\pi_{\rm data}}}
\def\gendist{{\pi_{\rm gen}}}
\def\priordist{{\pi_{\rm prior}}}
\def\realdist{{\pi_{\rm real}}}
\def\gaussiandist{{\mathcal{N}(\mathbf{0}, \mathbf{I})}}

\def\pgen{{p_{\rm gen}}}

\def\preal{{p_{\rm real}}}

\def\instableprob{{\mathcal{P}_M}}

\def\eqref#1{equation~\ref{#1}}

\def\1{\bm{1}}

\def\vh{{\bm{h}}}

\def\vn{{\bm{n}}}

\def\vu{{\bm{u}}}
\def\vv{{\bm{v}}}

\def\vx{{\bm{x}}}
\def\vy{{\bm{y}}}
\def\vz{{\bm{z}}}

\DeclareMathAlphabet{\mathsfit}{\encodingdefault}{\sfdefault}{m}{sl}
\SetMathAlphabet{\mathsfit}{bold}{\encodingdefault}{\sfdefault}{bx}{n}

\def\gA{{\mathcal{A}}}

\def\gE{{\mathcal{E}}}

\def\gN{{\mathcal{N}}}

\def\gR{{\mathcal{R}}}

\def\gU{{\mathcal{U}}}

\def\sR{{\mathbb{R}}}

\def\sZ{{\mathbb{Z}}}

\newcommand{\E}{\mathbb{E}}

\title{\PHenomenon~in Diffusion ODEs: \\An Explanation for Inaccurate Image Reconstruction}

\author{%
    {\bf Han Zhang$^{1*\ddag}$ \quad
    Jinghong Mao$^{*}$ \quad
    Shangwen Zhu$^{1}$ \quad
    Zhantao Yang$^{1\ddag}$}
    \\[1pt]
    {\bf Lianghua Huang$^{2}$ \quad
    Yu Liu$^{2}$ \quad
    Deli Zhao$^{3}$ \quad
    Ruili Feng$^{2}$ \quad
    Fan Cheng$^{1\dagger}$}
    \\[5pt]
    $^1$Shanghai Jiao Tong University \quad
    $^2$Tongyi Lab \quad
    $^3$Alibaba Group
}

\begin{document}

\maketitle

\nonumfootnote{
$^*$Equal contribution.
$^\ddag$Work performed during internship at Tongyi Lab.
$^\dagger$Corresponding author.
}
\nonumfootnote{
\ Emails:
Han Zhang $<$hzhang9617@gmail.com$>$,
Jinghong Mao $<$maojh1009@gmail.com$>$,
Fan Cheng $<$chengfan@sjtu.edu.cn$>$.
}

\vspace{-30pt}
\begin{abstract}

Diffusion reconstruction plays a critical role in various applications such as image editing, restoration, and style transfer.
In theory, the reconstruction should be simple -- it just inverts and regenerates images by numerically solving the Probability-Flow Ordinary Differential Equation (PF-ODE).
Yet in practice, noticeable reconstruction errors have been observed, which cannot be well explained by numerical errors.
In this work, we identify a deeper intrinsic property in the PF-ODE generation process, the \textit{\phenomenon}, that can further amplify the reconstruction errors.
The root of this \phenomenon~lies in \textit{the sparsity inherent in the generation distribution, which means that the probability is concentrated on scattered and small regions while the vast majority remains almost empty.}
To demonstrate the existence of \phenomenon~and its amplification on reconstruction error, we conduct experiments on both toy numerical examples and popular open-sourced diffusion models.
Furthermore, based on the characteristics of image data, we theoretically prove that the \phenomenon's~probability converges to one as the data dimensionality increases.
Our findings highlight the inherent challenges in diffusion-based reconstruction and can offer insights for future improvements.

\end{abstract}
\section{Introduction}
\label{sec:intro}

Diffusion models have rapidly emerged as a pivotal class of generative models, demonstrating superior performance, particularly in image generation~\citep{ho2020denoising,rombach2022high,saharia2022photorealistic,ramesh2022hierarchical,balaji2023ediffi,pernias2023wuerstchen,peebles2023scalable,podell2023sdxl,kawar2023imagic,li2023snapfusion,esser2024scaling,liu2024instaflow}. A fundamental technique within diffusion models is diffusion reconstruction, which comprises diffusion inversion~\citep{song2020denoising,ramesh2022hierarchical,chung2022come} and diffusion generation--both executed through numerical solving the Probability Flow Ordinary Differential Equations (PF-ODEs)~\citep{song2020score,lu2022dpm}.
Diffusion inversion\footnote{We follow the convention that uses \textit{inversion} indicating solving PF-ODE along the forward time~\citep{mokady2023null,wallace2023edict,zhang2024exact}.} first converts an image into an inverted noise, which is then used by diffusion generation process to reconstruct the original image. Diffusion reconstruction is crucial due to its extensive applications in downstream tasks, including image editing~\citep{gal2022image,hertz2022prompt}, restoration~\citep{xiao2024dreamclean}, and style transfer~\citep{su2022dual}.

\begin{figure}[t]
    \centering
    \includegraphics[width=\linewidth]{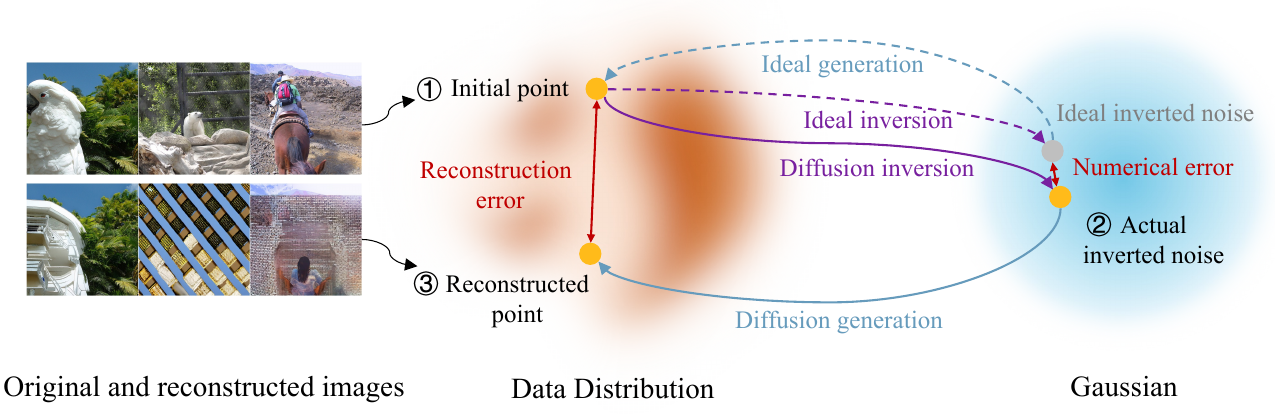}
    \vspace{-20pt}
    \caption{
    \textbf{\Phenomenon~amplifies diffusion reconstruction errors}.
    Given an \textit{initial point} in the \textit{data distribution}, the reconstruction process first undergoes \textit{diffusion inversion} to produce the \textit{actual inverted noise}.
    \textit{Numerical errors} introduced by the inversion will cause the actual inverted noise to deviate from the theoretically ideal inverted noise.
    When the generation process exhibits \phenomenon, these numerical errors are amplified, resulting in significantly larger \textit{reconstruction errors}.
    This amplification makes accurate reconstruction challenging.
    }
    \vspace{-15pt}
    \label{fig:error_amplification_concept}
\end{figure}

From the theoretical perspective, the reconstruction should be simple, which only needs to numerically solve the PF-ODE in two opposite directions.
Yet in practice, significant reconstruction errors can occasionally happen, as shown in \cref{fig:error_amplification_concept}.
Although previous works attribute the reconstruction inaccuracies to the inherent numerical discretization errors in solving PF-ODEs~\citep{wallace2023edict,wang2024belm,lin2024schedule,zhang2024exact},
it does not well explain such significant discrepancy.
The fundamental causes underlying these pronounced reconstruction discrepancies remain systematically unexplored.

In this work, we point out the existence of \textit{\phenomenon}~in the diffusion generation process, and demonstrate its amplification effect on the reconstruction errors.
Specifically, \phenomenon~characterizes scenarios where the diffusion generation process is sensitive to the initial noise.
Thus, any numerical error in the inverted noise will be amplified during the regeneration phase.
\cref{fig:error_amplification_concept} illustrates the process by which reconstruction errors arise.

To further demonstrate the presence of \phenomenon~from a theoretical perspective, we establish a lower bound on the probability of \phenomenon.
Note that here we choose to analyze the ideal diffusion reconstruction process, which is free from the influence of any numerical error and focused on the inherent property of PF-ODE. 
Based on reasonable assumptions about the real image distribution and the generation distribution of diffusion models, we demonstrate that when the data dimensionality increases as infinity, the probability of \phenomenon~in reconstruction tends to one!
Considering the high dimensionality of image data, such surprising asymptotic result provides theoretical justification for the \phenomenon~observed in image reconstruction within diffusion models.

\myparagraph{Mechanisms of \phenomenon.} For better understanding on the mechanism behind the \phenomenon~in diffusion reconstruction, we provide an intuitive illustration in \cref{fig:phenomenon-occurrence-concept}.
Our analysis reveals that \textit{the sparsity of the generation distribution plays a pivotal role in the emergence of \phenomenon}.
Here, the sparsity means that the generation distribution is concentrated in scattered, small regions, while the majority of regions possess low probability density.
Recall that the PF-ODE-based generation process actually builds a push-forward mapping from the Gaussian distribution to generation distribution.
According to the push-forward formula, the generation mapping must preserve probability.
Thus, \textit{suppose a region $A$ in the Gaussian distribution is mapped to a significantly lower density region $B$ in the generation distribution, the area of region $B$ will be extended to maintain the probability}. 
This expansion necessitates large gradients in the generation process, indicating the emergence of \phenomenon.
Meanwhile, during image reconstruction, the initial image is sampled from some underlying real distribution.
This real distribution is generally different from the generation distribution.
Such distribution discrepancy can lead to a non-negligible probability that the initial image resides in the low density region of the generation distribution, and thus make the \phenomenon~high probable.

\myparagraph{Main organization.}
In \cref{sec:preliminaries}, we will introduce preliminaries about PF-ODE and the definition of \phenomenon.
Subsequently in \cref{sec:phenomenon}, we will provide experimental evidence that \phenomenon~actually exists in diffusion generation, and then demonstrate its amplification on reconstruction error.
Finally in \cref{sec:analysis}, we further provide theoretical evidence on \phenomenon. Based on the characteristics of image distributions, we reveal that the sparsity of generation distribution can induce the \phenomenon, and theoretically prove that the \phenomenon~will almost surely occur during reconstruction for infinite dimensional images.
\section{Preliminaries} \label{sec:preliminaries}

\subsection{Diffusion models and probability flow ODE}

Diffusion models can generate samples under desired distribution $\gendist$ from noise under standard Gaussian distribution $\gaussiandist$~\citep{ho2020denoising,song2020score}.
The generation process can be achieved by numerically solving the PF-ODE~\citep{ANDERSON1982313,song2020score,liu2022flow}:
\begin{equation}
    \label{eq:pf-ode}
    \dd \vx_t / \dd t = \vv(\vx_t, t),\,\text{where }\,\vv(\vx, t) = \E[Z  - X | X_t = \vx],
\end{equation}
$t \in [0, 1]$, $\vv: \sR^n \times [0, 1] \to \sR^n$ is a time-dependent vector field defined by the conditional expectation in \cref{eq:pf-ode}.
In the expectation, the noise $Z \sim \gaussiandist$ and the data $X \sim \gendist$ are independently sampled, $X_t = (1 - t) X + t Z$ is a linear interpolation between them. 

Note that here we adopt the flow matching formulation to represent the PF-ODE~\citep{liu2022flow,lipman2022flow}, and the subsequent sections will consistently utilize this formulation.
Here, we adopt the flow matching formulation for two reasons: 1) the flow matching formulation is actually equivalent to conventional PF-ODEs when one distribution is set to Gaussian~\citep{song2020score,liu2022flow,albergo2023stochastic}, and 2) this formulation is widely adopted in recent top-tier text-to-image models~\citep{esser2024scaling,flux2024}.

\begin{figure}[t]
    \centering
    \includegraphics[width=0.7\linewidth]{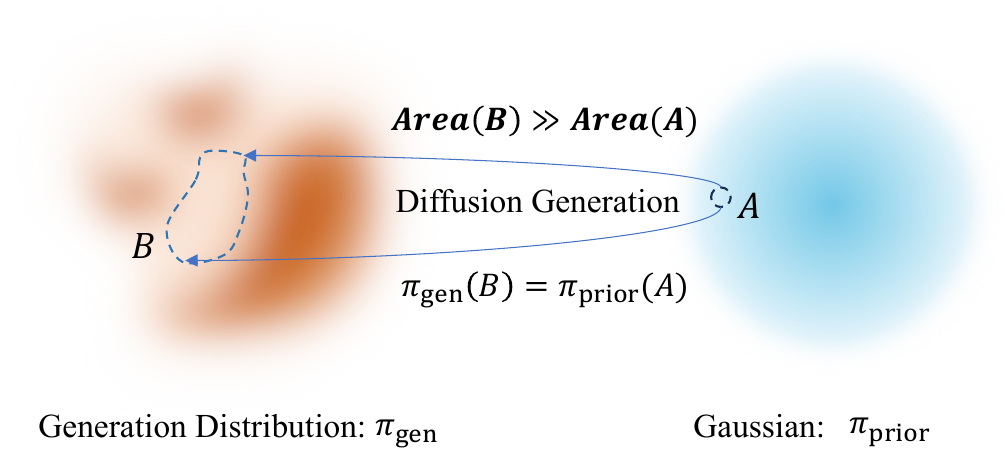}
    \vspace{-13pt}
    \caption{
    \textbf{Intuitive illustration of instability occurrence during the generation process.} For typical image data, \textit{the generation distribution is inherently sparse, meaning that most of the probability mass is concentrated in scattered, small regions, while the majority of the space has low probability.} In contrast, the prior distribution for the generation process, \textit{i.e.}, the Gaussian, concentrates its probability within a bounded region and is relatively uniformly distributed.
    According to the push-forward formula, \textit{the generation mapping must preserve probability measures, ensuring that the image of a probability region $A$ under the generation mapping $B$ satisfies $\gendist(B) = \priordist(A)$.} When $B$ falls into a low-density region of the generation distribution, significantly lower than the density at $A$, maintaining probability equality requires that the area (or more rigorously, the Lebesgue measure) of $B$ be much larger than that of $A$. This amplification effect results in large gradients during the generation process, indicating the emergence of \phenomenon. The probability of the \phenomenon~is analyzed in \cref{sec:analysis} in detail.
    }
    \vspace{-15pt}
    \label{fig:phenomenon-occurrence-concept}
\end{figure}

\myparagraph{Diffusion generation and inversion.}
Both diffusion generation and inversion solve the PF-ODE in \cref{eq:pf-ode}, but in opposite directions.
Diffusion generation mapping can produce data samples from initial noise, while diffusion inversion mapping yields an inverted noise from an initial data.
We formally define these two mappings as follows:
\begin{definition}[Diffusion generation and inversion mappings]
\label{def:diff-map}
The diffusion generation mapping $G: \sR^n \to \supp (\gendist)$ is defined as $G(\vz) = \vx_0$, where $\vx_0$ is the solution at $t=0$ of the ODE \cref{eq:pf-ode} with the initial value $\vx_1 = \vz$ at $t=1$.
Meanwhile, the diffusion inversion mapping can be denoted by $G^{-1}$, the inverse function of the generation mapping.
\end{definition}

\subsection{Definition of \phenomenon} \label{subsec:phenomenon-def}
For later analysis, we first provide the definition of \Intrinsiceffect~for general mapping $F$.
\begin{definition}[\Intrinsiceffect] \label{def:phenomenon}
Suppose $F: \sR^n \to \sR^n $ is a continuously differentiable mapping. For an input vector $\vx \in \sR^n$, if there exists another non-zero vector $\vu \in \sR^n$ such that
\begin{equation}
    \label{eq:intrinsic-effect-coef}
    \gE_F(\vx, \vu) \defeq \Vert J_F(\vx) \vu \Vert / \Vert \vu \Vert > 1,
\end{equation}
where $J_F(\vx) \in \sR^{n \times n}$ denotes the Jacobian matrix of $F$ evaluated at $\vx$, then we say that $F$ exhibits the \textit{\intrinsiceffect} at $\vx$ in the direction $\vu$. The scalar $\gE_F(\vx, \vu)$ is referred to as the \textit{\intrinsiceffectmetric} of $F$ at $\vx$ along the direction $\vu$.
\end{definition}

When the \intrinsiceffectmetric~exceeding one, it indeed indicates that the mapping $F$ amplifies any \textit{infinitesimal perturbation} along the specific direction $\vu$.
The magnitude of $\gE_F(\vx, \vu)$ quantitatively measures the sensitivity of $F$ to changes in the input direction $\vu$.
The following proposition will better illustrate this error amplification phenomenon for \textit{non-negligible perturbations}, which we refer to as \textit{\effect~effect}, when \intrinsiceffect~occurs:

\begin{restatable}[\Effect~effect]{proposition}{PROPeffect}\label{prop:effect}
Suppose $F: \sR^n \to \sR^n$ is a continuously differentiable mapping. For any $\vx,\vu \in \sR^n$, and any $\Delta>0$, there exists a non-zero perturbation $\vn \in \sR^n$ such that
\begin{equation} \label{eq:effect-coef}
    \gA_F(\vx, \vn) \defeq \frac{\Vert F(\vx + \vn) - F(\vx)\Vert}{\Vert \vn \Vert} \geq \frac{\mathcal E_F(\vx,\vu)}{1+\Delta},
\end{equation}
where $\gA_F(\vx, \vn)$ is referred to as the \textit{\effectmetric}. Furthermore, if $\mathcal E_F(\vx,\vu) > 1$, we have a $\vn$ that satisfies $\gA_F(\vx, \vn) > 1$, then we say that $F$ exhibits the \effect~at $\vx$ under perturbation $\vn$.
\end{restatable}
\section{\Phenomenon~and its amplification on reconstruction error} \label{sec:phenomenon}

\begin{figure*}[t]
    \centering
    \begin{overpic}[width=\linewidth]{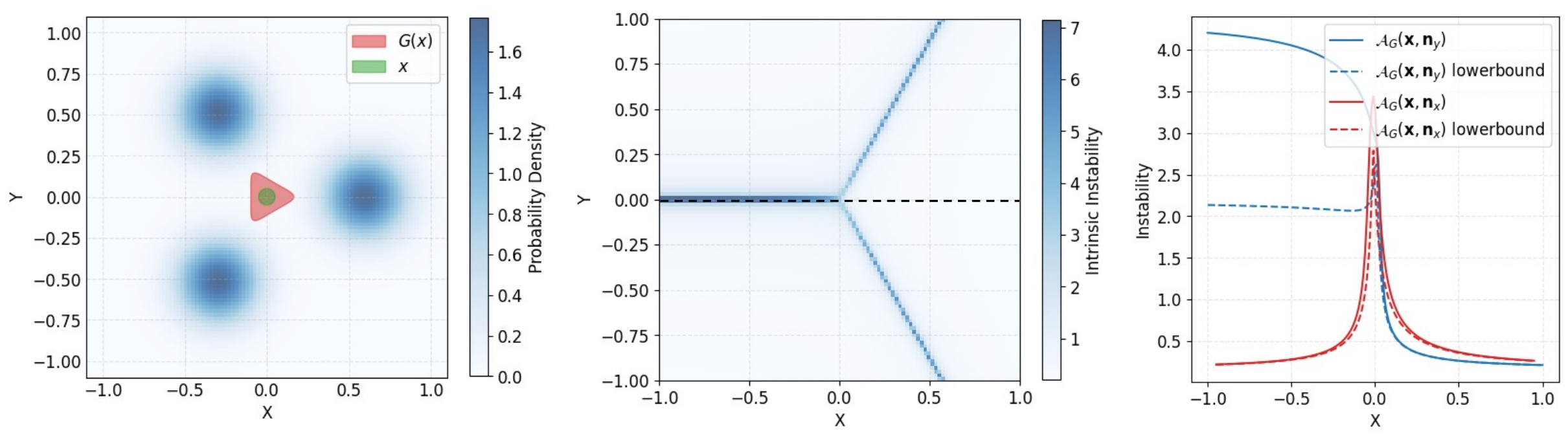}
\put(3,-2){(a) Generation distribution}
\put(38,-2){(b) \parbox[t]{5cm}{\Intrinsiceffectmetric\\ along $y$-axis}}
\put(75,-2){(c) \parbox[t]{5cm}{\Effectmetric\\ along $\vn_x$ and $\vn_y$}}
\end{overpic}
\vspace{8pt}
    \caption{\textbf{Experimental evidence on the existence of \phenomenon~for a two-dimensional generation distribution.} (a) visualizes the density function of the generation distribution, which is a mixture of Gaussians generated from a standard gaussian. The green region, after applying generation mapping, becomes a significantly larger red region. (b) visualizes the \intrinsiceffectmetric~along the $y$-axis. It is observed that the \intrinsiceffectmetric~along the $y$-axis in the central region significantly exceeds one, consistent with the position of the green region in (a). For each point on the dashed line in (b), we compute the \effectmetric~$\gA_G(\vx, \vn)$ of diffusion generation mapping $G$ under perturbation $\vn_x$ along $x$-axis and perturbation $\vn_y$ along $y$-axis, respectively. The results are shown in (c). The dashed lines show the lower bounds of $\gA_G(\vx, \vn)$ given in \cref{prop:effect}.}
    \label{fig:phenomenon-example}
    \vspace{-15pt}
\end{figure*}

In this section, we first empirically demonstrate the existence of \phenomenon~in the diffusion generation mapping $G$~in \Cref{subsec:phenomenon-by-exp}. Subsequently, we present experimental evidence corroborating the positive correlation between \phenomenon~and the reconstruction error in \Cref{subsec:recon-error}.

\subsection{Empirical evidence on the existence of \phenomenon} \label{subsec:phenomenon-by-exp}

In \cref{subsec:phenomenon-def}, we have defined the \intrinsiceffect~and introduced the \effect~effect. Here, we will demonstrate that the \effect~can exist in the generation process of diffusion models.

\myparagraph{Demonstration by a numerical case.} We present a numerical example illustrated in \cref{fig:phenomenon-example}. The density function of the diffusion generation distribution is depicted in \cref{fig:phenomenon-example}(a). The detailed settings are provided in the supplementary material.

In \cref{fig:phenomenon-example}(a), the red region is obtained by applying the diffusion generation mapping $G$ to the green region. It is evident that $G$ significantly expands the green region along the $y$-axis, indicating that $G$ is highly sensitive to perturbations along the $y$-axis within the green region. \cref{fig:phenomenon-example}(b) displays the \intrinsiceffectmetric~of $G$ along the $y$-axis. It can be observed that the \intrinsiceffectmetric~along the $y$-axis exceeds one in the central region, indicating that the \effect~effect can appear in this area. Moreover, this region coincides with the green region in \cref{fig:phenomenon-example}(a), demonstrating the relationship between the \intrinsiceffect~and the amplification effect.

We further analyze the \effectmetric~$\gA_G(\vx, \vn)$ along the dashed line in \cref{fig:phenomenon-example}(b), as well as the lower bound provided by \cref{prop:effect}. The curves of $\gA_G(\vx, \vn)$ \textit{vs.} $x$ are shown in \cref{fig:phenomenon-example}(c). We observe that in the central region, the \effectmetrics~along $x$ and $y$ axis, \textit{i.e.}, $\gA_G(\vx, \vn_x)$ and $\gA_G(\vx, \vn_y)$, are greater than one, further indicating the presence of \phenomenon.

\LATER{more examples }

\subsection{\Phenomenon~amplifies reconstruction error} \label{subsec:recon-error}

This subsection demonstrates that \phenomenon~in diffusion generation can amplify reconstruction errors.
We begin by formally defining the reconstruction error, and provide a theoretical support on the possibility of large reconstruction error.
Then we demonstrate the amplification effect by correlation analyses using both numerical cases and practical diffusion models.

\myparagraph{Reconstruction error.}
Diffusion reconstruction involves two primary steps for a given data $\vx \in \sR^n$:
\begin{enumerate}[leftmargin=2em,topsep=0ex,itemsep=0ex]
    \setlength{\parskip}{0ex}
    \item The diffusion inversion process defined by \cref{def:diff-map} is numerically solved from the given data $\vx$, which yields an inverted noise $\hat{\vz} = \widehat{G}^{-1}(\vx)$; and then
    \item the diffusion generation process defined by \cref{eq:pf-ode} is numerically solved from $\hat{\vz}$, resulting in a reconstructed data $\hat{\vx} = \widehat{G}(\hat{\vz})$,
\end{enumerate}
where $\widehat{G}$ denotes the approximated diffusion generation mapping by numerically solving PF-ODE in \cref{eq:pf-ode}, and $\widehat{G}^{-1}$ denotes the approximated diffusion inversion mapping, similarly.

The reconstruction error is quantified as the discrepancy between the original data $\vx$ and the reconstructed data $\hat{\vx}$. It can be defined as:
\begin{equation}
\label{eq:recon-error}
    \gR(\vx) = \frac{1}{\sqrt{n}}\Vert \vx - \hat{\vx} \Vert_2 = \frac{1}{\sqrt{n}}\Vert \vx - \widehat{G}(\hat{\vz}) \Vert_2 = \frac{1}{\sqrt{n}}\Vert \vx - \widehat{G}(\widehat{G}^{-1}(\vx)) \Vert_2.
\end{equation}

\begin{figure*}[t]
    \centering
    \begin{overpic}[width=\linewidth]{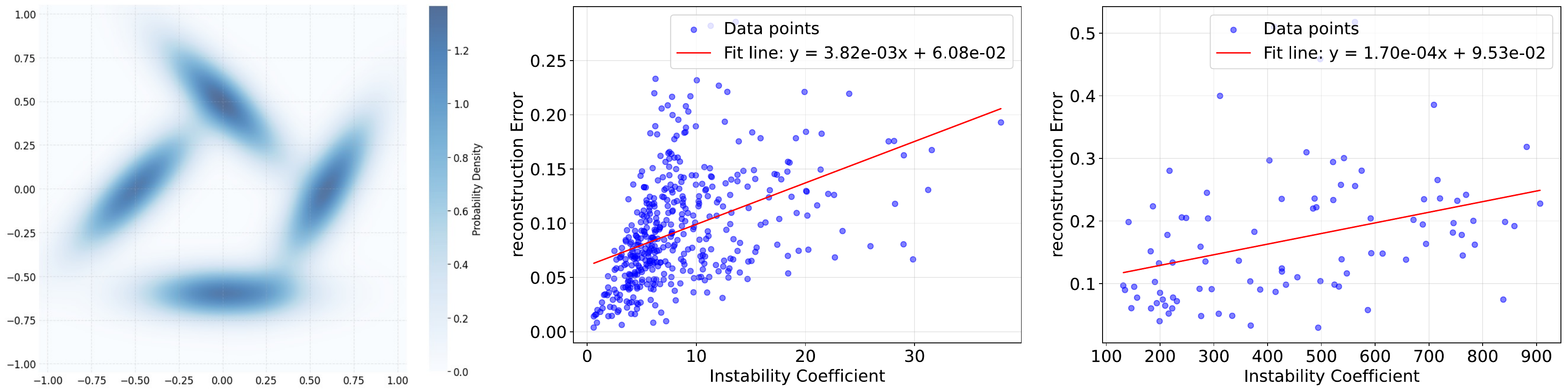}
\put(2,-3){(a) Generation distribution.}
\put(38,-3){\parbox[t]{5cm}{(b) Reconstruction error\\ of the numerical example.}}
\put(73,-3){\parbox[t]{5cm}{(c) Reconstruction error\\ of Stable Diffusion 3.5.}}
\end{overpic}
\vspace{9pt}
    \caption{\textbf{Positive correlation between the reconstruction error and \effectmetric~in a numerical case and Stable Diffusion 3.5~\cite{esser2024scaling}.} In numerical experiments, the diffusion model's generation distribution is a two-dimensional mixture of Gaussians. (a) shows the probability density function. (b) displays the relationship between reconstruction error $\gR(\vx)$ and instability coefficient obtained from reconstruction of different initial values $\vx$. (c) further illustrates the relationship between reconstruction error and instability coefficient in the Stable Diffusion 3.5 model.}
    \label{fig:recon-error}
    \vspace{-10pt}
\end{figure*}

To analyze the causes of reconstruction error, it is essential to recognize two primary factors:
\begin{enumerate}[leftmargin=2em,topsep=0ex,itemsep=0ex]
    \setlength{\parskip}{0ex}
    \item \textit{Discrepancy in inverted noise}: The inverted noise $\hat{\vz} = \widehat{G}^{-1}(\vx)$ obtained during the diffusion inversion process deviates from the ideal noise $\vz \sim \gaussiandist$. This discrepancy arises due to numerical discretization errors inherent in solving the PF-ODE~\cref{eq:pf-ode} and its inverse. As a result, the numerically inverted noise does not perfectly match the ideal one, introducing an initial error into the regeneration pipeline.
    
    \item \textit{Amplification by \phenomenon}: The presence of \phenomenon~in the diffusion generation mapping $\widehat{G}$ can further exacerbate the above discrepancy in the inverted noise, leading to a substantial reconstruction error $\gR(\vx)$.
\end{enumerate}
\vspace{5pt}

The following theorem further demonstrates that both of the aforementioned two factors contribute to the reconstruction error, even if the diffusion generation mapping $G$ is ideal with infinite numerical precision. A detailed proof is provided in the supplementary material.
\begin{restatable}[Risk of large reconstruction error]{theorem}{THMreconerror}\label{thm:recon-error}
    For any data sample $\vx \in \sR^n$, consider the diffusion reconstruction process consisting of a numerically approximated diffusion inversion under Euler method and a precise diffusion generation process $G$.
    Let $\hat{\vz}$ denote the numerically inverted noise, $\vz$ denote the ideal inverted noise, and $\hat{\vx}$ denote the reconstructed data.
    Suppose $\mathcal{E}_G(\vz',\hat{\vz} - \vz) > C$ for some $C > 0$ and all $\vz' \in \{t \vz + (1 - t) \hat{\vz} : t \in [0, 1]\}$. Then, the upper bound $\gU$ of reconstruction error $\mathcal R(\vx)$ satisfies:
    \begin{equation}
        \gU \geq  \underbrace{h {M_2} \frac{(C-1)}{2\log C}}_{\substack{\text{Numerical error} \\ \text{in inverted noise}}} \;\;\cdot \underbrace{C \vphantom{\frac{(C-1)}{2\log C}}}_{\substack{\text{Instability} \\ \text{amplification}}},
    \end{equation}
    where $h$ represents the step size in the numerical solution, and $M_2$ is the estimated upper bound term for the local truncation error of the Euler method.
\end{restatable}

\myparagraph{A numerical case.}
To verify the amplification effect of \phenomenon~on reconstruction error, we conduct numerical experiments on a diffusion model with a simple generation distribution. Then, we analyze the correlation between the reconstruction error and the \effectmetric. The density function of the generation distribution is illustrated in \cref{fig:recon-error}(a). Detailed settings are provided in the supplementary material.

The experimental results are presented in \cref{fig:recon-error}(b). We observe a significant correlation between the reconstruction error $\gR(\vx)$ and the averaged \effectmetric. As the averaged \effectmetric~increases, the reconstruction error exhibits a noticeable upward trend. This positive correlation corroborates our hypothesis that the \phenomenon~amplifies reconstruction errors.

\myparagraph{Verification on text-to-image diffusion models.}
In practical text-to-image diffusion models, significant reconstruction errors can occasionally occur. To verify that these reconstruction errors are still related to the \effectmetric, we performed a correlation analysis between the reconstruction error $\gR(\vx)$ and the expansion coefficient $\gE$ on the Stable Diffusion 3.5 model~\citep{esser2024scaling}. Detailed experimental settings are provided in the supplementary material.

The experimental results are illustrated in \cref{fig:recon-error}(c). It demonstrate that the reconstruction error $\gR(\vx)$ is significantly positively correlated with the averaged \effectmetric~$\gE$ in Stable Diffusion 3.5. This finding validates our hypothesis that \phenomenon~contributes to reconstruction inaccuracies in real-world diffusion models.
In the supplementary material, we provide more visual examples of reconstruction failure cases.
\section{Probabilistic guarantee on the occurrence of \phenomenon} \label{sec:analysis}
In this section, provide theoretical evidence on \phenomenon~during the ideal diffusion reconstruction process.
First, we establish a general probability lower bound without specific assumptions in \cref{subsec:analysis-prob-lower-bound}.
For further in-depth analysis, we then discuss the characteristics of the distributions involved in diffusion reconstruction and propose corresponding assumptions in \cref{subsec:prob-analysis-setting}.
Finally, based on the distribution assumptions, we provide a theoretical analysis indicating the non-zero probability of \phenomenon~in \cref{subsec:prob-analysis}.

\subsection{Probability lower bound as intuitive evidence for \phenomenon}
\label{subsec:analysis-prob-lower-bound}

\myparagraph{Definition of additional instability metric.}
Before the formal analysis, we first define \textit{the geometric average of \intrinsiceffectmetric} as an indicator of \phenomenon.

\begin{definition}[Geometric average of \intrinsiceffectmetric] \label{def:geo-avg-effect-coef}
Suppose $F: \sR^n \to \sR^n$ is a continuously differentiable mapping. For any $\vx \in \sR^n$, we define the \textit{geometric average of \intrinsiceffectmetric}~as
    \begin{equation}
        \bar{\gE}_F(\vx) \defeq \left(\prod_{i=1}^n \gE_F(\vx, \vu_i)\right)^{\frac 1 n},
    \end{equation}
where $\vu_1, \cdots, \vu_n$ is a set of right singular vectors of the Jacobian matrix $J_F(\vx)$. Note that $\bar{\gE}_F(\vx)$ is also the geometric average of all singular values of $F$'s Jacobian $J_F(\vx)$.
\end{definition}
When $\bar{\gE}_F(\vx) > 1$, at least one of $\gE_F(\vx, \vu_i)$ is also greater than one, thereby indicating the \phenomenon.

\myparagraph{Lower bound of \phenomenon's probability.}
Applying the above instability metric to the diffusion generation mapping $G$, we can now present the following theorem:
\begin{restatable}[Probability lower bound of \phenomenon]{theorem}{THMproplowerbound}
\label{thm:prob-lower-bound}
    Suppose $G$ is the ideal diffusion generation mapping defined in \cref{def:diff-map}, whose generation distribution is denoted as $\gendist$.
    Let $G^{-1}$ denote its inverse mapping, which represents the ideal diffusion inversion mapping.
    Further suppose we sample the initial data $\vx$ from some underlying real distribution $\realdist$ for reconstruction. 
    Then, for any $M > 0$, we have
    \begin{align}
        \label{eq:prob-lower-bound}
        \instableprob &\defeq \realdist\left(\{\vx: \bar{\gE}_G(G^{-1}(\vx)) > M \}\right) \geq 1 - \epsilon - \delta,\\
        \label{eq:epsilon}
        \epsilon &\defeq \realdist\left(\left\{\vx: \pgen(\vx) 
        \geq \frac{1}{(2\pi M^2)^{\frac n 2}}e^{-\frac{2n+3\sqrt{2n}}{2}}\right\}\right), \\
        \label{eq:delta}
        \delta &\defeq \realdist(\{\vx: \Vert G^{-1}(\vx) \Vert^2 > 2n+3\sqrt{2n}\}),
    \end{align}
    where $\instableprob \defeq \realdist\left(\{\vx: \bar{\gE}_G(G^{-1}(\vx)) > M \}\right)$ represents the probability of \phenomenon~in the ideal diffusion reconstruction if $M > 1$, and $\pgen$ denotes the probability density function of $\gendist$.

    More specifically, $\instableprob$ describes the probability that the geometric average of \intrinsiceffectmetric~$\bar{\gE}_G(\vz)$ is greater than the threshold $M$ on the inverted noise $\vz=G^{-1}(\vx)$, where $\vx$ is sampled from the underlying real distribution $\realdist$.
\end{restatable}
Here, $\epsilon$ denotes the probability that the density function $\pgen(x)$ of $\gendist$ at a real data $x$ greater than a threshold $ \frac{1}{(2\pi M^2)^{\frac n 2}}e^{-\frac{2n+3\sqrt{2n}}{2}}$ , while $\delta$ represents the probability that the inverted noise $G^{-1}(\vx)$ of a real data significantly deviates from the center.
When both $\epsilon$ and $\delta$ are small enough, and $M>1$, we can conclude that $\instableprob > 0$, and the \phenomenon~of the diffusion generation mapping is probable to happen during the ideal reconstruction process.

In the next two subsections, we will first make reasonable assumptions based on the characteristics of image data and practical considerations in \cref{subsec:analysis-prob-lower-bound}, and then provide an asymptotic guarantee for the occurrence of \phenomenon~in \cref{subsec:prob-analysis}.

\subsection{Setting discussions for in-depth analysis} \label{subsec:prob-analysis-setting}

To further analyze the probability $\instableprob$ of \phenomenon~during the diffusion reconstruction process, we first need to make reasonable assumptions about the two probability distributions, $\realdist$ and $\gendist$, involved in \cref{thm:prob-lower-bound}. In the following sections, we will discuss these two distributions separately.

\subsubsection{Discussions on the properties of $\realdist$.}
The distribution $\realdist$ represents the distribution of images to be reconstructed and is not constrained by the diffusion model itself. We can assume that any real-world image may be used for reconstruction, so we refer to $\realdist$ as the \textit{real distribution}.

Based on this analysis and the inherent properties of image data, the real distribution $\realdist$ is characterized by the following two features:
\begin{enumerate}[leftmargin=2em,topsep=0ex,itemsep=0ex]
    \setlength{\parskip}{0ex}
    \item \textbf{Support as a cube}: The support of $\realdist$ is a hypercube in $\mathbb{R}^n$. After normalization, it can be assumed that $\supp(\realdist) = [0,1]^n$. This is because each pixel in an image has bounded values, typically normalized to the interval $[0, 1]$, ensuring that images in $n$-dimensional space reside within the hypercube $[0,1]^n$.

    \item \textbf{Positive minimum density}: The density function $\preal$ of $\realdist$ has a positive minimum value across its support. This implies that every point within $[0,1]^n$ corresponds to a possible image, including those that may be uncommon or represent noise-like structures. Although such images have extremely low probability density, they remain potential members of $\realdist$.
\end{enumerate}

\begin{assumption}[$\realdist$ -- Support as a cube] \label{assump:real-dist-cube}
The support of $\realdist$ is the $n$-dimensional hypercube $[0,1]^n$, i.e., $\supp(\realdist) = [0,1]^n$.
\end{assumption}

\begin{assumption}[$\preal$ -- Each pixel get a chance] \label{assump:real-density-lower-bound}
The density function $\preal$ of $\realdist$ is continuous on $\supp(\realdist)$, and $\preal$ further satisfies $\forall \vx \in [0,1]^n$, $\preal(\vx) \geq C_0 > 0$.
\end{assumption}

\subsubsection{Discussions on the properties of $\gendist$.}
The generation distribution $\gendist$ refers to the distribution of samples generated by the diffusion model when it starts from Gaussian noises.
For complex and high-dimensional image data, \textit{$\gendist$ typically exhibits sparsity characteristics, meaning it contains several high-probability regions surrounded by areas of lower probability}\TOCITE{}. In theoretical analyses, the mixture of Gaussians is a common choice for modeling such multi-modal distributions\TOCITE{}. However, in this work, to better differentiate between high and low probability regions, we construct a more general distribution model named as the \textit{mixture of Gaussian neighbors}. The assumption on the density function $\pgen$ of $\gendist$ is provided here:

\begin{assumption}[$\pgen$ -- Mixture of Gaussian neighbors] \label{assump:gen-dist-mixture}
The density function $\pgen$ can be expressed as $\pgen = \sum_{i=1}^m a_i f_i * g_{w_i}$,
where each $f_i$ is a probability density function supported on the open set $O_i$, $g_{w_i} = \gN(0, w_i^2 I)$ is a Gaussian kernel with standard deviation $w_i$, $*$ denotes convolution operator, $m \in \sZ^+$, and the coefficients $a_i$ satisfy $a_i > 0$ for all $i = 1, \dots, m$ and $\sum_{i=1}^m a_i = 1$.
\end{assumption}
In this density function, each density function $f_i$ on open set $O_i$ captures the high density region, while the convolution with the Gaussian kernel $g_{w_i}$ ensures that the surrounding areas have smoothly decreasing probabilities, and thus models the low density region. This approach allows for a flexible representation of complex, multi-modal distributions commonly encountered in high-dimensional image data. The following theorem further demonstrate the approximation capability of the mixture of Gaussian neighbors.

\begin{restatable}[Mixture of Gaussian neighbors are Universal Approximators]{theorem}{THMuniversalapprox} \label{thm:mix-of-gaussian-neighbor-approx}
    The set of density functions $\{p : p = \sum_{i=1}^m a_i f_i * g_{w_i},\text{$f_i$ is compactly supported continuous function}\}$ is a dense subset of continuous density function in both $L^2$ metric~\citep{folland1999real} and L\'{e}vy-Prokhorov metric~\citep{billing}.
\end{restatable}

\myparagraph{Sparsity assumption group.} Combining the model's actual training process, the meaning of the mixture of Gaussian neighbors can be understood as follows:
In high-dimensional space for image data, the training samples is sparse and finite due to the limited capability to collect image. When the model is sufficiently trained on the training set and possesses certain generalization capabilities, it can generate a small neighborhood around each training sample with relatively high probability. Thus, each $O_i$ corresponds to the neighborhood that can be generated around each training sample, and $m$ represents the number of training samples.
Considering the sparsity of high-dimensional images and the model's limited generalization capabilities, we make the following group of assumptions:
\begin{assumption}[Sparsity assumption group]\label{assump:sparsity-group}
\ 
\begin{enumerate}[leftmargin=2em,topsep=0ex,itemsep=0ex]
    \setlength{\parskip}{0ex}
\item (Finite training set) We assume that $m$ has an upper bound as the dimension $n$ increases. For convenience, we directly assume that $m$ is a constant.
    
\item (Sparse data in high-dimensional space)
The neighborhoods $O_i$ and $O_j$ do not overlap for any $i \neq j$, and there exists a minimum positive distance $d_{\min} > 0$ between any two distinct neighborhoods $O_i$ and $O_j$. For convenience, we define $\bar{d}_\min \defeq d_{\min} / \sqrt{n}$.

\item (Finite generalization capability)
Each neighborhood $O_i$ is included in a hypercube $B_i$ of side length $b_i \ll 1$.
    
\item (Low probability region) For each $f_i$, the probability outside its corresponding $O_i$ is upper bounded by a constant $1 - \alpha_i$, \textit{i.e.}, $\int_{O_i^C} f_i * g_{w_i} (\vx) \dd \vx < 1 - \alpha_i < 1$.
\end{enumerate}
\end{assumption}
    
These assumptions formalize the intuition that in high-dimensional spaces, training samples are finite, sparse, and each training sample can generate a distinct, small neighborhood without overlap. Additionally, the probability density outside these neighborhoods is uniformly bounded, ensuring that low-density regions do not dominate the generation process.

\subsection{Asymptotic proof of almost sure \phenomenon} \label{subsec:prob-analysis}
Building upon the assumptions on the relevant distributions, in this subsection, we will present a more in-depth asymptotic analysis on $\instableprob$, demonstrating that \phenomenon~occurs almost surely during diffusion reconstruction when the dimensionality $n$ tends to infinity.
Regarding real-world high-dimensional images, this asymptotic result implies that \phenomenon~will occur with a non-negligible probability within the reconstruction process.

Recall that in \cref{thm:prob-lower-bound}, we have proved that the \phenomenon~probability $\instableprob \geq 1 - \epsilon - \delta$, where $M$ is the threshold of an \phenomenon~indicator $\bar{\gE}_G$ defined in \cref{def:geo-avg-effect-coef}.
To prove a almost sure \phenomenon~as $n \to \infty$, it is sufficient to prove that $\epsilon \to 0$ and $\delta \to 0$ when $M > 1$.
\cref{thm:main} exactly supports this claim. The proof is provided in the supplementary material.
\begin{restatable}{theorem}{THMmain} \label{thm:main}
    Consider $\epsilon$ and $\delta$ defined in \cref{eq:epsilon,eq:delta} in \cref{thm:prob-lower-bound}. Suppose that \cref{assump:real-dist-cube,assump:real-density-lower-bound,assump:gen-dist-mixture} and \cref{assump:sparsity-group}--the sparsity assumption group--hold.
    When $n \to \infty$, if $M$ satisfies $M < M_0$, we have
    \begin{align}
        \label{eq:epsilon-lim}
        \epsilon &\defeq \realdist\left(\left\{\vx: \pgen(\vx) 
        \geq \frac{1}{(2\pi M^2)^{\frac n 2}}e^{-\frac{2n+3\sqrt{2n}}{2}}\right\}\right) \to 0, \\
        \label{eq:delta-lim}
        \delta &\defeq \realdist(\{\vx: \Vert G^{-1}(\vx) \Vert^2 > 2n+3\sqrt{2n}\}) \to 0,
    \end{align}
    where $M_0 \defeq \min_{1\leq i\leq m}\exp \left(\frac{1}{8}\frac{\bar{d}_\min^2}{w_i}-\ln \frac{1}{w_i} + 2 + 3\sqrt{\frac{2}{n}} \right) \to \infty$.
    Thus, for any fixed $M > 1$, we can conclude that the \phenomenon~probability $\instableprob \to 1$ as $n \to \infty$.
\end{restatable}
The above theorem implies that, for a sufficiently large dimension $n$, such as the case of high-dimensional image data, the \phenomenon~probability can be high, significantly greater than zero.

\myparagraph{Intuitive explanations.}
To better understand \cref{thm:main} and the mechanism underlying this highly-probable \phenomenon, we revisit the roles of $\epsilon$ and $\delta$ in the lower bound of the \phenomenon~probability $\instableprob$.
The term $\epsilon \defeq \realdist\left(\left\{\vx: \pgen(\vx) 
        \geq \frac{1}{(2\pi M^2)^{\frac n 2}}e^{-\frac{2n+3\sqrt{2n}}{2}}\right\}\right)$ quantifies the probability that samples from $\realdist$ avoid low probability density regions of $\gendist$. As discussed in \cref{subsec:prob-analysis-setting}, the sparsity of $\gendist$ ensures that most regions exhibit negligible probability density, thereby driving $\epsilon$ toward zero.
Meanwhile, another term $\delta \defeq \realdist(\{\vx: \Vert G^{-1}(\vx) \Vert^2 > 2n+3\sqrt{2n}\})$ represents the probability that the inverted noise $G^{-1}(\vx)$ deviates from the Gaussian's center beyond a squared threshold $2n+3\sqrt{2n}$ — a bound that scales as infinity when dimensionality $n$ increases.
It is known that high-dimensional Gaussian samples in $\sR^n$ concentrate near a sphere with squared radius $n$~\TOCITE{}.
This concentration indicates that significant deviations are of small probability, which implies a small $\delta$.
Consequently, both two terms can be small, implying a positive \phenomenon~probability lower bound $1 - \epsilon - \delta$.
A detailed derivation is provided in the supplementary material.
\section{Related works} \label{sec:related-works}

The concepts of diffusion inversion and reconstruction were initially introduced by~\citet{song2020denoising} as an application of DDIM, and was then used in many image editing tasks~\citep{hertz2022prompt,chung2022improving,su2022dual,tumanyan2023plug}. However, many studies have identified limitations in their effectiveness for text-to-image diffusion models, particularly in terms of reconstruction quality~\citep{dhariwal2021diffusion,ho2022classifier,zhang2024exact,mokady2023null,wallace2023edict,hong2024exact,wang2025belm,dai2024erddci}.
To improve the reconstruction accuracy, some tuning-based methods are proposed to align the re-generation trajectory to the inversion trajectory\citep{mokady2023null,dong2023prompt}.
As it is recognized that numerical discretization error is the direct cause of reconstruction inaccuracy, high-order ODE solvers for diffusion models offer alternatives to mitigate this issue~\citep{lu2022dpm,lu2022dpmplus,karras2022elucidating}.
Additionally, some works introduce auxiliary variables to ensure the numerical reversibility of the diffusion inversion process~\citep{wallace2023edict,zhang2024exact,wang2024belm}.
However, it has been noted that such methods can be unstable, as small perturbations in the inverted noise may lead to significant deviations in the reconstructed image~\citep{ju2024pnp}.
Meanwhile, there remains a lack of comprehensive analysis on the underlying mechanism: why the reconstruction discrepancy can sometimes be sufficiently significant.
Different from previous works that proposes new methods, this work attempts to dive into the mechanism and advance the understanding of diffusion models.
\section{Conclusion}
\label{sec:conclusion}

In this paper, we identify the \phenomenon~phenomenon as an amplifier on the error observed in diffusion-based image reconstruction. Through rigorous theoretical analysis and comprehensive experimental validations, we demonstrate that \phenomenon~leads to the amplification of numerical perturbations during the diffusion generation process, thus increasing the reconstruction error.
Moreover, we investigate the underlying causes of \phenomenon~in the diffusion ODE generation process, demonstrating that the inherent sparsity in diffusion generation distribution can cause \phenomenon.
Meanwhile, we theoretically prove that the \phenomenon~will almost surely occur during reconstruction for infinite dimensional images.
Our work elucidates a critical source of reconstruction inaccuracies.
Addressing \phenomenon~will be essential for advancing the reliability of generative models in practical settings.

\myparagraph{Limitations.}
While our analysis establishes \phenomenon~in diffusion reconstruction as a principal contributor to significant errors, we do not preclude the coexistence of additional error mechanisms. Furthermore, the probabilistic proof of \phenomenon~under our realistic assumptions leaves open the formal characterization of necessary and sufficient conditions for \phenomenon~emergence.

\myparagraph{Broader Impacts.}
This paper aims to advance the understanding on the mechanism of probably large reconstruction error in diffusion models. While this research primarily contributes to the theoretical advancement in the field, it may also inspire future practical improvements and applications, though no immediate societal impacts are identified for specific emphasis at this time.

{
    \bibliographystyle{abbrvnat}
    \bibliography{ref}

\begin{thebibliography}{51}
\providecommand{\natexlab}[1]{#1}
\providecommand{\url}[1]{\texttt{#1}}
\expandafter\ifx\csname urlstyle\endcsname\relax
  \providecommand{\doi}[1]{doi: #1}\else
  \providecommand{\doi}{doi: \begingroup \urlstyle{rm}\Url}\fi

\bibitem[Albergo et~al.(2023)Albergo, Boffi, and Vanden-Eijnden]{albergo2023stochastic}
M.~S. Albergo, N.~M. Boffi, and E.~Vanden-Eijnden.
\newblock Stochastic interpolants: A unifying framework for flows and diffusions.
\newblock \emph{arXiv preprint arXiv:2303.08797}, 2023.

\bibitem[Anderson(1982)]{ANDERSON1982313}
B.~D. Anderson.
\newblock Reverse-time diffusion equation models.
\newblock \emph{Stochastic Processes and their Applications}, 12\penalty0 (3):\penalty0 313--326, 1982.
\newblock ISSN 0304-4149.
\newblock \doi{https://doi.org/10.1016/0304-4149(82)90051-5}.
\newblock URL \url{https://www.sciencedirect.com/science/article/pii/0304414982900515}.

\bibitem[Balaji et~al.(2023)Balaji, Nah, Huang, Vahdat, Song, Zhang, Kreis, Aittala, Aila, Laine, Catanzaro, Karras, and Liu]{balaji2023ediffi}
Y.~Balaji, S.~Nah, X.~Huang, A.~Vahdat, J.~Song, Q.~Zhang, K.~Kreis, M.~Aittala, T.~Aila, S.~Laine, B.~Catanzaro, T.~Karras, and M.-Y. Liu.
\newblock e{D}iff-{I}: Text-to-image diffusion models with an ensemble of expert denoisers, 2023.
\newblock URL \url{https://arxiv.org/abs/2211.01324}.

\bibitem[Billingsley(1999)]{billing}
P.~Billingsley.
\newblock \emph{Convergence of probability measures}.
\newblock Wiley Series in Probability and Statistics: Probability and Statistics. John Wiley \& Sons Inc., New York, second edition, 1999.
\newblock ISBN 0-471-19745-9.
\newblock A Wiley-Interscience Publication.

\bibitem[{Black Forest Labs}(2024)]{flux2024}
{Black Forest Labs}.
\newblock Flux.1.
\newblock \url{https://blackforestlabs.ai}, 2024.
\newblock URL \url{https://blackforestlabs.ai}.
\newblock Accessed: 2024-12-03.

\bibitem[Chung et~al.(2022{\natexlab{a}})Chung, Sim, Ryu, and Ye]{chung2022improving}
H.~Chung, B.~Sim, D.~Ryu, and J.~C. Ye.
\newblock Improving diffusion models for inverse problems using manifold constraints.
\newblock \emph{NeurIPS}, 35:\penalty0 25683--25696, 2022{\natexlab{a}}.

\bibitem[Chung et~al.(2022{\natexlab{b}})Chung, Sim, and Ye]{chung2022come}
H.~Chung, B.~Sim, and J.~C. Ye.
\newblock Come-closer-diffuse-faster: Accelerating conditional diffusion models for inverse problems through stochastic contraction.
\newblock In \emph{CVPR}, pages 12413--12422, 2022{\natexlab{b}}.

\bibitem[Dai et~al.(2024)Dai, Zhang, Chen, Yang, and Luo]{dai2024erddci}
J.~Dai, Y.~Zhang, S.~Chen, J.~Yang, and L.~Luo.
\newblock Erddci: Exact reversible diffusion via dual-chain inversion for high-quality image editing.
\newblock \emph{arXiv preprint arXiv:2410.14247}, 2024.

\bibitem[Dhariwal and Nichol(2021)]{dhariwal2021diffusion}
P.~Dhariwal and A.~Nichol.
\newblock Diffusion models beat {GAN}s on image synthesis.
\newblock \emph{NeurIPS}, 2021.

\bibitem[Dong et~al.(2023)Dong, Xue, Duan, and Han]{dong2023prompt}
W.~Dong, S.~Xue, X.~Duan, and S.~Han.
\newblock Prompt tuning inversion for text-driven image editing using diffusion models.
\newblock In \emph{CVPR}, pages 7430--7440, 2023.

\bibitem[Esser et~al.(2024)Esser, Kulal, Blattmann, Entezari, M{\"u}ller, Saini, Levi, Lorenz, Sauer, Boesel, et~al.]{esser2024scaling}
P.~Esser, S.~Kulal, A.~Blattmann, R.~Entezari, J.~M{\"u}ller, H.~Saini, Y.~Levi, D.~Lorenz, A.~Sauer, F.~Boesel, et~al.
\newblock Scaling rectified flow transformers for high-resolution image synthesis.
\newblock In \emph{ICML}, 2024.

\bibitem[Folland(2013)]{folland2013real}
G.~Folland.
\newblock \emph{Real Analysis: Modern Techniques and Their Applications}.
\newblock Pure and Applied Mathematics: A Wiley Series of Texts, Monographs and Tracts. Wiley, 2013.
\newblock ISBN 9781118626399.
\newblock URL \url{https://books.google.com.hk/books?id=wI4fAwAAQBAJ}.

\bibitem[Folland(1999)]{folland1999real}
G.~B. Folland.
\newblock \emph{Real analysis: modern techniques and their applications}, volume~40.
\newblock John Wiley \& Sons, 1999.

\bibitem[{Franzen, R}(1999)]{kodak}
{Franzen, R}.
\newblock Kodak lossless true color image suite.
\newblock \url{http://r0k. us/graphics/kodak}, 1999.
\newblock URL \url{http://r0k. us/graphics/kodak}.
\newblock Accessed: 2024-12-03.

\bibitem[Gal et~al.(2022)Gal, Alaluf, Atzmon, Patashnik, Bermano, Chechik, and Cohen-Or]{gal2022image}
R.~Gal, Y.~Alaluf, Y.~Atzmon, O.~Patashnik, A.~H. Bermano, G.~Chechik, and D.~Cohen-Or.
\newblock An image is worth one word: Personalizing text-to-image generation using textual inversion.
\newblock \emph{arXiv preprint arXiv:2208.01618}, 2022.

\bibitem[Hertz et~al.(2022)Hertz, Mokady, Tenenbaum, Aberman, Pritch, and Cohen-Or]{hertz2022prompt}
A.~Hertz, R.~Mokady, J.~Tenenbaum, K.~Aberman, Y.~Pritch, and D.~Cohen-Or.
\newblock Prompt-to-prompt image editing with cross attention control.
\newblock \emph{arXiv preprint arXiv:2208.01626}, 2022.

\bibitem[Ho and Salimans(2022)]{ho2022classifier}
J.~Ho and T.~Salimans.
\newblock Classifier-free diffusion guidance.
\newblock \emph{arXiv preprint arXiv:2207.12598}, 2022.

\bibitem[Ho et~al.(2020)Ho, Jain, and Abbeel]{ho2020denoising}
J.~Ho, A.~Jain, and P.~Abbeel.
\newblock Denoising diffusion probabilistic models.
\newblock In \emph{NeurIPS}, 2020.

\bibitem[Hong et~al.(2024)Hong, Lee, Jeon, Bae, and Chun]{hong2024exact}
S.~Hong, K.~Lee, S.~Y. Jeon, H.~Bae, and S.~Y. Chun.
\newblock On exact inversion of dpm-solvers.
\newblock In \emph{CVPR}, 2024.

\bibitem[Ikeda and Watanabe(2014)]{ikeda2014stochastic}
N.~Ikeda and S.~Watanabe.
\newblock \emph{Stochastic Differential Equations and Diffusion Processes}.
\newblock North-Holland Mathematical Library. North Holland, 2014.
\newblock ISBN 9781483296159.
\newblock URL \url{https://books.google.com.hk/books?id=QZbOBQAAQBAJ}.

\bibitem[Ju et~al.(2024)Ju, Zeng, Bian, Liu, and Xu]{ju2024pnp}
X.~Ju, A.~Zeng, Y.~Bian, S.~Liu, and Q.~Xu.
\newblock Pnp inversion: Boosting diffusion-based editing with 3 lines of code.
\newblock In \emph{ICLR}, 2024.

\bibitem[Karras et~al.(2022)Karras, Aittala, Aila, and Laine]{karras2022elucidating}
T.~Karras, M.~Aittala, T.~Aila, and S.~Laine.
\newblock Elucidating the design space of diffusion-based generative models.
\newblock \emph{arXiv preprint arXiv:2206.00364}, 2022.

\bibitem[Kawar et~al.(2023)Kawar, Zada, Lang, Tov, Chang, Dekel, Mosseri, and Irani]{kawar2023imagic}
B.~Kawar, S.~Zada, O.~Lang, O.~Tov, H.~Chang, T.~Dekel, I.~Mosseri, and M.~Irani.
\newblock Imagic: Text-based real image editing with diffusion models, 2023.
\newblock URL \url{https://arxiv.org/abs/2210.09276}.

\bibitem[Li et~al.(2023)Li, Wang, Jin, Hu, Chemerys, Fu, Wang, Tulyakov, and Ren]{li2023snapfusion}
Y.~Li, H.~Wang, Q.~Jin, J.~Hu, P.~Chemerys, Y.~Fu, Y.~Wang, S.~Tulyakov, and J.~Ren.
\newblock Snapfusion: Text-to-image diffusion model on mobile devices within two seconds, 2023.
\newblock URL \url{https://arxiv.org/abs/2306.00980}.

\bibitem[Lin et~al.(2024)Lin, Wang, Wang, An, Chen, Liu, Tian, Dai, Wang, and Wang]{lin2024schedule}
H.~Lin, M.~Wang, J.~Wang, W.~An, Y.~Chen, Y.~Liu, F.~Tian, G.~Dai, J.~Wang, and Q.~Wang.
\newblock Schedule your edit: A simple yet effective diffusion noise schedule for image editing.
\newblock \emph{arXiv preprint arXiv:2410.18756}, 2024.

\bibitem[Lin et~al.(2014)Lin, Maire, Belongie, Hays, Perona, Ramanan, Doll{\'a}r, and Zitnick]{lin2014microsoft}
T.-Y. Lin, M.~Maire, S.~Belongie, J.~Hays, P.~Perona, D.~Ramanan, P.~Doll{\'a}r, and C.~L. Zitnick.
\newblock Microsoft coco: Common objects in context.
\newblock In \emph{ECCV}, pages 740--755. Springer, 2014.

\bibitem[Lipman et~al.(2023)Lipman, Chen, Ben-Hamu, Nickel, and Le]{lipman2022flow}
Y.~Lipman, R.~T. Chen, H.~Ben-Hamu, M.~Nickel, and M.~Le.
\newblock Flow matching for generative modeling.
\newblock In \emph{ICLR}, 2023.

\bibitem[Liu et~al.(2023)Liu, Gong, and Liu]{liu2022flow}
X.~Liu, C.~Gong, and Q.~Liu.
\newblock Flow straight and fast: Learning to generate and transfer data with rectified flow.
\newblock In \emph{ICLR}, 2023.

\bibitem[Liu et~al.(2024)Liu, Zhang, Ma, Peng, and Liu]{liu2024instaflow}
X.~Liu, X.~Zhang, J.~Ma, J.~Peng, and Q.~Liu.
\newblock Instaflow: One step is enough for high-quality diffusion-based text-to-image generation, 2024.
\newblock URL \url{https://arxiv.org/abs/2309.06380}.

\bibitem[Lu et~al.(2022{\natexlab{a}})Lu, Zhou, Bao, Chen, Li, and Zhu]{lu2022dpm}
C.~Lu, Y.~Zhou, F.~Bao, J.~Chen, C.~Li, and J.~Zhu.
\newblock Dpm-solver: A fast ode solver for diffusion probabilistic model sampling in around 10 steps.
\newblock \emph{arXiv preprint arXiv:2206.00927}, 2022{\natexlab{a}}.

\bibitem[Lu et~al.(2022{\natexlab{b}})Lu, Zhou, Bao, Chen, Li, and Zhu]{lu2022dpmplus}
C.~Lu, Y.~Zhou, F.~Bao, J.~Chen, C.~Li, and J.~Zhu.
\newblock Dpm-solver++: Fast solver for guided sampling of diffusion probabilistic models.
\newblock \emph{arXiv preprint arXiv:2211.01095}, 2022{\natexlab{b}}.

\bibitem[Mokady et~al.(2023)Mokady, Hertz, Aberman, Pritch, and Cohen-Or]{mokady2023null}
R.~Mokady, A.~Hertz, K.~Aberman, Y.~Pritch, and D.~Cohen-Or.
\newblock Null-text inversion for editing real images using guided diffusion models.
\newblock In \emph{CVPR}, pages 6038--6047, 2023.

\bibitem[Peebles and Xie(2023)]{peebles2023scalable}
W.~Peebles and S.~Xie.
\newblock Scalable diffusion models with transformers.
\newblock In \emph{ICCV}, pages 4195--4205, 2023.

\bibitem[Pernias et~al.(2023)Pernias, Rampas, Richter, Pal, and Aubreville]{pernias2023wuerstchen}
P.~Pernias, D.~Rampas, M.~L. Richter, C.~J. Pal, and M.~Aubreville.
\newblock Wuerstchen: An efficient architecture for large-scale text-to-image diffusion models, 2023.
\newblock URL \url{https://arxiv.org/abs/2306.00637}.

\bibitem[Podell et~al.(2023)Podell, English, Lacey, Blattmann, Dockhorn, M{\"u}ller, Penna, and Rombach]{podell2023sdxl}
D.~Podell, Z.~English, K.~Lacey, A.~Blattmann, T.~Dockhorn, J.~M{\"u}ller, J.~Penna, and R.~Rombach.
\newblock Sdxl: improving latent diffusion models for high-resolution image synthesis.
\newblock \emph{arXiv preprint arXiv:2307.01952}, 2023.

\bibitem[Ramesh et~al.(2022)Ramesh, Dhariwal, Nichol, Chu, and Chen]{ramesh2022hierarchical}
A.~Ramesh, P.~Dhariwal, A.~Nichol, C.~Chu, and M.~Chen.
\newblock Hierarchical text-conditional image generation with clip latents.
\newblock \emph{arXiv preprint arXiv:2204.06125}, 2022.

\bibitem[Rombach et~al.(2022)Rombach, Blattmann, Lorenz, Esser, and Ommer]{rombach2022high}
R.~Rombach, A.~Blattmann, D.~Lorenz, P.~Esser, and B.~Ommer.
\newblock High-resolution image synthesis with latent diffusion models.
\newblock In \emph{CVPR}, 2022.

\bibitem[Saharia et~al.(2022)Saharia, Chan, Saxena, Li, Whang, Denton, Ghasemipour, Ayan, Mahdavi, Lopes, et~al.]{saharia2022photorealistic}
C.~Saharia, W.~Chan, S.~Saxena, L.~Li, J.~Whang, E.~Denton, S.~K.~S. Ghasemipour, B.~K. Ayan, S.~S. Mahdavi, R.~G. Lopes, et~al.
\newblock Photorealistic text-to-image diffusion models with deep language understanding.
\newblock \emph{arXiv preprint arXiv:2205.11487}, 2022.

\bibitem[Schneider(1993)]{schneider1993convex}
R.~Schneider.
\newblock \emph{Convex Bodies: The Brunn-Minkowski Theory}.
\newblock Encyclopedia of Mathematics and its Applications. Cambridge University Press, 1993.
\newblock ISBN 9780521352208.
\newblock URL \url{https://books.google.com.hk/books?id=2QhT8UCKx2kC}.

\bibitem[Song et~al.(2020)Song, Meng, and Ermon]{song2020denoising}
J.~Song, C.~Meng, and S.~Ermon.
\newblock Denoising diffusion implicit models.
\newblock \emph{arXiv preprint arXiv:2010.02502}, 2020.

\bibitem[Song et~al.(2021)Song, Sohl-Dickstein, Kingma, Kumar, Ermon, and Poole]{song2020score}
Y.~Song, J.~Sohl-Dickstein, D.~P. Kingma, A.~Kumar, S.~Ermon, and B.~Poole.
\newblock Score-based generative modeling through stochastic differential equations.
\newblock \emph{ICLR}, 2021.

\bibitem[Stein and Shakarchi(2003)]{stein2003fourier}
E.~Stein and R.~Shakarchi.
\newblock \emph{Fourier Analysis: An Introduction}.
\newblock Princeton University Press, 2003.
\newblock ISBN 9780691113845.
\newblock URL \url{https://books.google.com.hk/books?id=I6CJngEACAAJ}.

\bibitem[Su et~al.(2022)Su, Song, Meng, and Ermon]{su2022dual}
X.~Su, J.~Song, C.~Meng, and S.~Ermon.
\newblock Dual diffusion implicit bridges for image-to-image translation.
\newblock \emph{arXiv preprint arXiv:2203.08382}, 2022.

\bibitem[Tao(2021)]{tao2021introduction}
T.~Tao.
\newblock \emph{An Introduction to Measure Theory}.
\newblock Graduate Studies in Mathematics. American Mathematical Society, 2021.
\newblock ISBN 9781470466404.
\newblock URL \url{https://books.google.com.hk/books?id=k0lDEAAAQBAJ}.

\bibitem[Tumanyan et~al.(2023)Tumanyan, Geyer, Bagon, and Dekel]{tumanyan2023plug}
N.~Tumanyan, M.~Geyer, S.~Bagon, and T.~Dekel.
\newblock Plug-and-play diffusion features for text-driven image-to-image translation.
\newblock In \emph{CVPR}, pages 1921--1930, 2023.

\bibitem[von Platen et~al.(2022)von Platen, Patil, Lozhkov, Cuenca, Lambert, Rasul, Davaadorj, Nair, Paul, Berman, Xu, Liu, and Wolf]{von-platen-etal-2022-diffusers}
P.~von Platen, S.~Patil, A.~Lozhkov, P.~Cuenca, N.~Lambert, K.~Rasul, M.~Davaadorj, D.~Nair, S.~Paul, W.~Berman, Y.~Xu, S.~Liu, and T.~Wolf.
\newblock Diffusers: State-of-the-art diffusion models.
\newblock \url{https://github.com/huggingface/diffusers}, 2022.

\bibitem[Wallace et~al.(2023)Wallace, Gokul, and Naik]{wallace2023edict}
B.~Wallace, A.~Gokul, and N.~Naik.
\newblock Edict: Exact diffusion inversion via coupled transformations.
\newblock In \emph{CVPR}, pages 22532--22541, 2023.

\bibitem[Wang et~al.(2024)Wang, Yin, Dong, Zhu, Zhang, Zhao, Qian, and Li]{wang2024belm}
F.~Wang, H.~Yin, Y.~Dong, H.~Zhu, C.~Zhang, H.~Zhao, H.~Qian, and C.~Li.
\newblock Belm: Bidirectional explicit linear multi-step sampler for exact inversion in diffusion models.
\newblock \emph{arXiv preprint arXiv:2410.07273}, 2024.

\bibitem[Wang et~al.(2025)Wang, Yin, Dong, Zhu, Zhao, Qian, Li, et~al.]{wang2025belm}
F.~Wang, H.~Yin, Y.-J. Dong, H.~Zhu, H.~Zhao, H.~Qian, C.~Li, et~al.
\newblock Belm: Bidirectional explicit linear multi-step sampler for exact inversion in diffusion models.
\newblock \emph{NeurIPS}, 2025.

\bibitem[Xiao et~al.(2024)Xiao, Feng, Zhang, Liu, Yang, Zhu, Fu, Zhu, Liu, and Zha]{xiao2024dreamclean}
J.~Xiao, R.~Feng, H.~Zhang, Z.~Liu, Z.~Yang, Y.~Zhu, X.~Fu, K.~Zhu, Y.~Liu, and Z.-J. Zha.
\newblock Dreamclean: Restoring clean image using deep diffusion prior.
\newblock In \emph{ICLR}, 2024.

\bibitem[Zhang et~al.(2024)Zhang, Lewis, and Kleijn]{zhang2024exact}
G.~Zhang, J.~P. Lewis, and W.~B. Kleijn.
\newblock Exact diffusion inversion via bidirectional integration approximation.
\newblock In \emph{ECCV}, pages 19--36. Springer, 2024.

\end{thebibliography}
}


\newpage
\appendix
\thispagestyle{empty}
\tableofappendix
\setcounter{page}{1}

\newcommand{\AppendixPrefix}{A}
\setcounter{section}{0}
\renewcommand{\thefigure}{\AppendixPrefix\arabic{figure}}
\setcounter{figure}{0}
\renewcommand{\thetable}{\AppendixPrefix\arabic{table}}
\setcounter{table}{0}
\renewcommand{\theequation}{\AppendixPrefix\arabic{equation}}
\setcounter{equation}{0}

\section{Pseudocode for diffusion reconstruction}

\begin{algorithm}[H]
\caption{Diffusion reconstruction}
\label{alg:recon}
\begin{algorithmic}[1]

\Statex \textbf{Input}: Initial data $\vx$, PF-ODE vector field $\vv$, ODE solver $\mathrm{ODESolver}$
\Statex \textbf{Output}: Reconstructed data $\hat{\vx}$
\Statex \# $\mathrm{ODESolver}$ gets four arguments as below
\Statex \# $\mathrm{ODESolver}(\text{initial value}, \text{derivative function},\text{initial time},\text{end time})$
\Statex
\Statex \# Diffusion inversion from $t=0$ to $t=1$
\State $\hat{\vz} \gets \mathrm{ODESolver}(\vx, \vv, 0, 1)$
\Statex \# Diffusion regeneration from $t=1$ to $t=0$
\State $\hat{\vx} \gets \mathrm{ODESolver}(\hat{\vz}, \vv, 1, 0)$
\end{algorithmic}
\end{algorithm}

To further illustrate the PF-ODE-based diffusion reconstruction procedure, we provide the pseudocode in \cref{alg:recon}.
The pseudocode demonstrates the core logic of the reconstruction porcess, utilizing an ODE solver that could be any numerical method.

\section{Additional experiments} \label{supp:exp}

\subsection{Reconstruction error amplification verified by second-order ODE solver} \label{suppsub:second-order-recon}

\begin{figure}[htbp]
    \centering
    \begin{subfigure}[b]{0.45\textwidth}
        \centering
        \includegraphics[width=\linewidth]{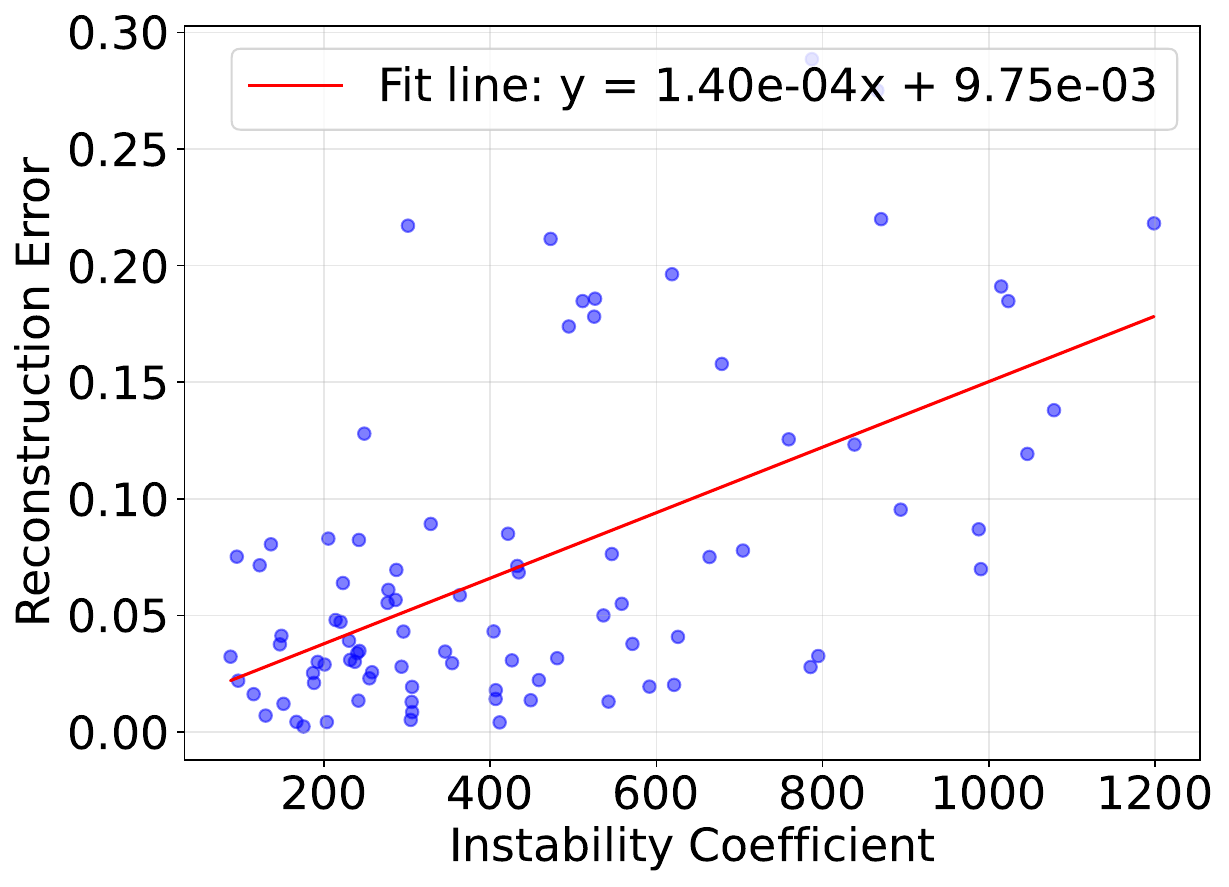} 
        \caption{Stable Diffusion 3.5 Medium~\citep{esser2024scaling}}
        \label{fig:recon_error_sd35_step500_heun}
    \end{subfigure}
    \hfill
    \begin{subfigure}[b]{0.45\textwidth}
        \centering
        \includegraphics[width=\linewidth]{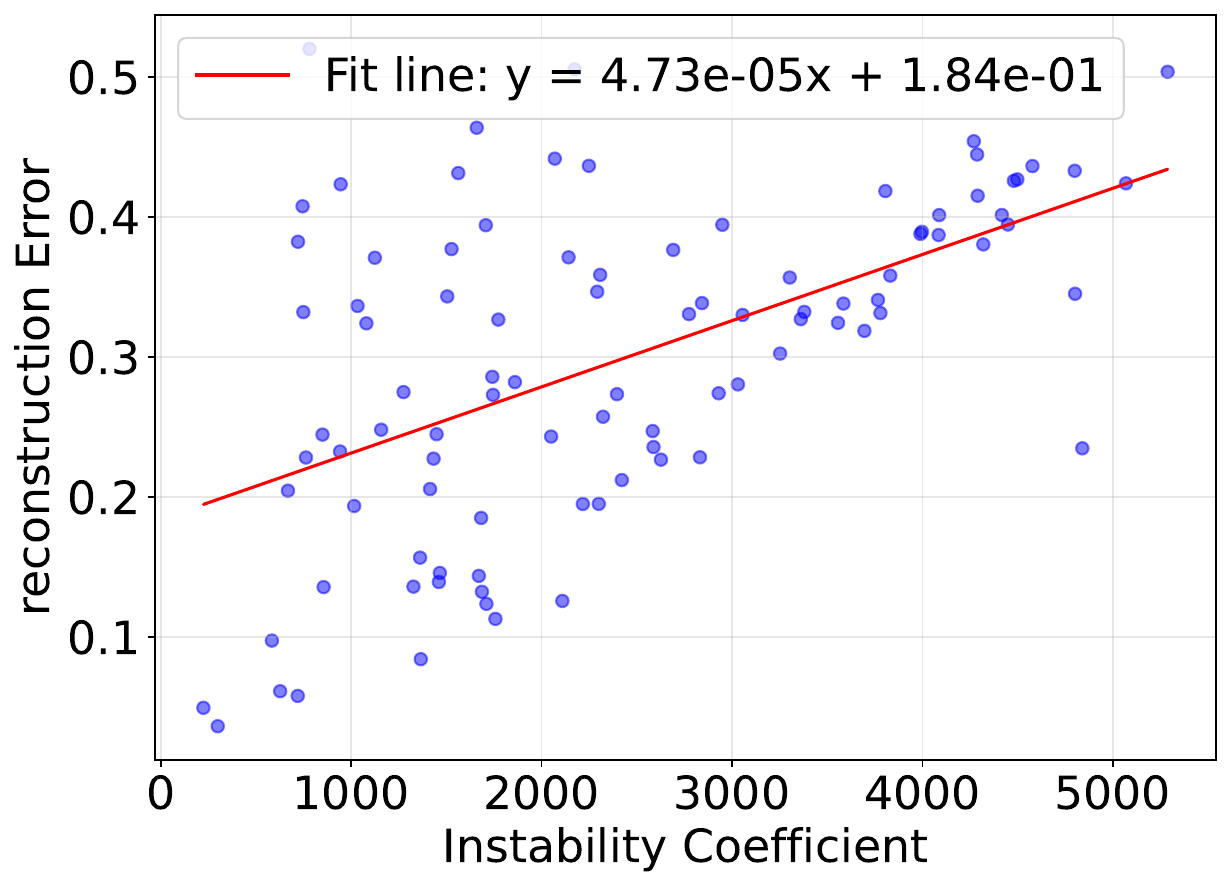} 
        \caption{Stable Diffusion 3.5 Large~\citep{esser2024scaling}}
        \label{fig:recon_error_sd35_large_step500_heun_euler}
    \end{subfigure}
    \begin{subfigure}[b]{0.45\textwidth}
        \centering
        \includegraphics[width=\linewidth]{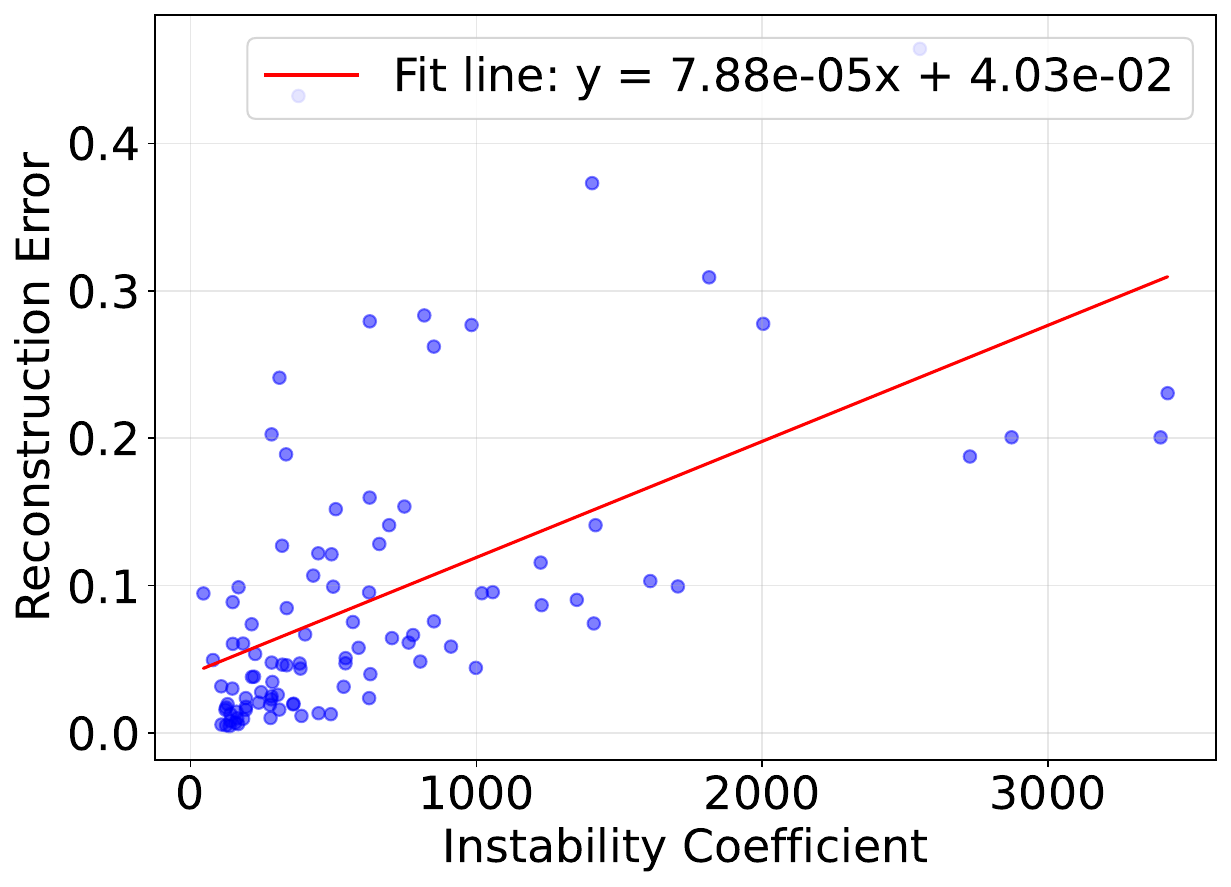} 
        \caption{FLUX~\citep{flux2024}}
        \label{fig:recon_error_flux_step500_heun_euler}
    \end{subfigure}
    \caption{
    \textbf{Positive correlation between the reconstruction error and \effectmetric~verified by second-order Heun ODE solver.} All of the three text-to-image diffusion models exhibit such positive correlation.
    Experimental details can be found in \cref{supp:exp-setting}.
    }
    \label{fig:recon_error_sd35_flux}
\end{figure}

In \cref{subsec:recon-error}, we have already demonstrated the positive correlation between the reconstruction error and the \effectmetric~of diffusion generation mapping.
Both the numerical example and the results of Stable Diffusion 3.5 Medium~\citep{esser2024scaling} in \cref{fig:recon-error} imply that the occurrence of \phenomenon~can amplify the reconstruction error.
Note that the Euler solver is adopted in~\cref{subsec:recon-error}.

Here we further emphasize that \textbf{this positive correlation between the reconstruction error and the \effectmetric~is an inherent property of PF-ODE, and does not rely on specific numerical method}.
To support this, we provide more experimental evidence under the second-order Heun ODE solver, rather than the first-order Euler solver in \cref{subsec:recon-error}.
The results are illustrated in \cref{fig:recon_error_sd35_flux}.
Here we choose three popular different text-to-image diffusion models--Stable Diffusion 3.5 Medium / Large~\citep{esser2024scaling} and FLUX~\citep{flux2024}.
It can be observed that the positive correlation between the reconstruction error and the \effectmetric~again appears in all of these three models.

\section{Discussions on the implication of \phenomenon}
Our findings reveal that diffusion generation process can exhibit \phenomenon~when reconstructing a target data, such as image. Specifically, the reconstructed result may vary dramatically under tiny perturbations on the inverted noise, indicating a high sensitivity that undermines the ability to consistently reproduce the original sample.
This behavior sheds light on an important limitation of diffusion models’ generalization capability. While diffusion models are often lauded for their strong performance in generating realistic samples, the observed \phenomenon~implies that they may not have fully learned the underlying structure of the data distribution.

From a theoretical perspective, as discussed in \cref{sec:analysis}, such \phenomenon~can be attributed to the sparsity of the generation distribution: the generation distribution concentrates its probability mass on scattered, small regions in the high-dimensional space, failing to provide robust coverage over the broader data manifold.
And this sparsity highlights the model’s limited capacity to generalize beyond the concentrated regions. These insights underscore the need for further research into improving how diffusion models capture the global data distribution, mitigate distributional sparsity, and enhance the stability of their generation process. By addressing these issues, future advancements can potentially expand the scope and reliability of diffusion-based generative models.

\section{Detailed proofs} \label{supp:proofs}

\subsection{Proof of \cref{prop:effect}} \label{suppsub:proof-thm-effect}
For convenience, we first repeat \cref{prop:effect} to be proved here.
\PROPeffect*

\begin{proof}
    Given that the operator norm of the Jacobian matrix $J_F(\vx)$ satisfies $\|J_F(\vx)\|=\sup_{\vu'\in \mathbb R^n}\frac{\|J_F(\vx)\vu'\|}{\|\vu'\|}\geq\mathcal E_F(\vx,\vu)$ for any $\vx,\vu \in \sR^n$, there exists a unit vector $\vh$ such that
    \begin{equation} \label{eq:norm-ineq}
    \|J_F(\vx)\vh\|\geq \mathcal E_F(\vx,\vu).
    \end{equation}
    Let $\vn_t=t\vh$, we have
    \begin{equation}
        F(\vx+\vn_t)-F(\vx)=J_F(\vx) \vn_t + R(\vn_t),
    \end{equation}
    where $R(\vn_t)$ is the remainder. Then $\gA_F(\vx, \vn_t)$ is bounded below by:
    \begin{align}
        \gA_F(\vx, \vn_t) &\defeq \frac{\|F(\vx+\vn_t)-F(\vx) \|}{\|\vn_t\|}\\
        &\geq \frac{\|J_F(\vx) \vn_t \|-\|R(\vn_t)\|}{\|\vn_t\|}\\
        &\geq \frac{t\mathcal E_F(\vx,\vu)-\|R(\vn_t)\|}{\|\vn_t\|}\\
        &=\mathcal E_F(\vx,\vu)-\frac{\|R(\vn_t)\|}{\|\vn_t\|},
    \end{align}
    where the first $\geq$ is by triangle inequality, and the second $\geq$ is by \cref{eq:norm-ineq}.
    
    Since $F$ is differentiable, the remainder $R(\vn_t)$ satisfies $\lim_{t\rightarrow0}\frac{\|R(\vn_t)\|}{\|\vn_t\|}=0$.
    Then for $\epsilon = \frac{\Delta}{1+\Delta}\mathcal E_F(\vx,\vu)$, there exist $\delta>0$ and $t\leq \delta$ such that $\frac{\|R(\vn_t)\|}{\|\vn_t\|}\leq \epsilon$, and this can lead to the conclusion
    \begin{equation}
        \frac{\|F(\vx+\vn_t)-F(\vx) \|}{\|\vn_t\|}\geq \mathcal E_F(\vx,\vu)-\frac{\|R(\vn_t)\|}{\|\vn_t\|}\geq \frac{\mathcal E_F(\vx,\vu)}{1+\Delta}.
    \end{equation}
    Furthermore, by the arbitrariness of $\Delta$, when $\mathcal E_F(\vx,\vu)>1$, there always exists a $\Delta$ satisfies $\mathcal E_F(\vx,\vu)>1+\Delta$. Thus $\gA_F(\vx, \vn_t) \geq \frac{\mathcal E_F(\vx,\vu)}{1+\Delta} > 1$.
\end{proof}

\subsection{Proof of \cref{thm:recon-error}} \label{suppsub:proof-thm-recon-error}
For convenience, we first repeat \cref{thm:recon-error} to be proved here.
\THMreconerror*
\begin{proof}
    The reconstruction process involves two sequential steps. First, the probability flow ODE in \cref{eq:pf-ode} is integrated forward from $t=0$ to $t=1$ using the Euler method, yielding the noise corresponding to the image. Then, the same ODE is integrated backward from $t=1$ to $t=0$, starting from the noise, to recover the reconstructed image. We use the reconstruction process to evaluate the Lipschitz constant of $\vv(\vx,t)$ and to analyze the error between the inversion and reconstruction procedures.
    The form of the probability flow ODE solved using the Euler method in the inversion process is:
    \begin{equation}
        \vx_{t_{n+1}}=\vx_{t_n}+(t_{n+1}-t_n) \vv(\vx_{t_n},t_n).
    \end{equation}
    The error analysis of Euler's numerical solution involves the estimation of local truncation error (LTE) and global truncation error (GTE), which together determine the method's overall accuracy.

    The local truncation error measures the error introduced in a single step of the Euler method. Denoting the time step $t_{n+1}-t_n$ as $h$, by expanding the true solution $\vx(t)$ at time $t_{n+1}$ using a Taylor series around $t_n$, we obtain:
    \begin{equation}
        \vx (t_{n+1})=\vx (t_n)+h \vx'(t_n)+\frac{h^2}{2}\vx''(\xi),
    \end{equation}
    where $\xi\in[t_n,t_{n+1}]$. Since $\vx'(t_n) = \vv(\vx_{t_n},t_n)$, the local truncation error is given by assuming $\vx_{t_n}=\vx(t_n)$:
    \begin{equation}
        \textit{LTE} \defeq \vx (t_{n+1})-\vx_{t_{n+1}}=\frac{h^2}{2}\vx''(\xi).
    \end{equation}
    If the second derivative of $\vx$ is bounded, that is, $|\vx''|\leq M_2$, then the error bound is 
    \begin{equation}
        \textit{LTE}\leq \frac{1}{2}M_2 h^2.
    \end{equation}
    The global truncation error measures the accumulated error over multiple steps, and the total error can be estimated by summing up the contributions from each step.
    Let $L$ denote the Lipschitz constant of $\vv$, then the global truncation error $E_n$ satisfies the following recurrence relation:
    \begin{align}
        E_{n+1} \defeq& \|\vx_{t_n+1}-\vx(t_{n+1})\|\\
        =&\|(\vx_{t_n}+h\vv(\vx_{t_n},t_n)-(\vx(t_n)+h\vx'(t_n)+\frac{h^2}{2}\vx''(\xi))\|\\
        \leq&\|\vx_{t_n}-\vx(t_n)\|+h\|\vv(\vx_{t_n},t_n)-\vv(\vx({t_n}),t_n)\|+\textit{LTE} \\
        \leq& E_n + h L \cdot E_n+\frac{1}{2}M_2 h^2.
    \end{align}
    This yields 
    \begin{equation}
        E_n \leq \frac{hM_2}{2L}\left(e^{L}-1\right).
    \end{equation}
    While considering the Gronwall inequality for the precise generation process, we can obtain that for two PF-ODE solutions $\vx_t,\,\vy_t$ with distinct initial conditions $\vx_1,\,\vy_1$ at time $t=1$, the exact solutions at time $t=0$ satisfy the following property:
    \begin{equation}
        \|\vy_{0}-\vx_0\|\leq \|\vx_1-\vy_1\|e^{L}.
    \end{equation}
    This equation will give a upper bound to the \intrinsiceffectmetric~$ \gE_G(\vx, \vu)$:
    \begin{equation}
    \label{eq:intrinsic-metric-bound-by-L}
        \sup_{\vx,\vu}\gE_G(\vx, \vu) \leq e^{L}.
    \end{equation}
    From \cref{eq:intrinsic-metric-bound-by-L} and the assumption $\gE_G(\vz^\prime, \hat{\vz} - \vz) > C$,
    we can infer that $L\geq \log C$.
    Substituting into our previous discussion on the global truncation error, we have:
    \begin{equation}
        \mathcal{U}\geq \frac{h M_2}{2L}(e^{L}-1) \cdot C \geq h M_2 \frac{C-1}{2 \log C} \cdot C.
    \end{equation}
\end{proof}

\subsection{Proof of \cref{thm:prob-lower-bound}} \label{suppsub:additional-lemmas}
To prove \cref{thm:prob-lower-bound}, we first present two lemmas.
\begin{lemma}\label{push forward}
    For any ODE with a unique solution:
    \begin{equation}\label{lemmaode}
        \frac{\dd\vx_t}{\dd t}=\vv(\vx_t,t),
    \end{equation}
    its solution can be expressed as $\vx_t=\phi_t(\vx_0)$ for $t \in [0, 1]$ with the initial condition $\vx_0$ at $t=0$. Let $p_t(\vx)$ be a probability density function that satisfies $\int_{\phi_t^{-1}(A)} p_0(\vx)\dd\vx=\int_A p_t(\vx)\dd\vx$ for any measurable set $A$.
    Then the geometric average of \intrinsiceffectmetric~for mapping $\phi_1$, \textit{i.e.}, $\bar {\mathcal {E}}_{\phi_1}(\vx)$ defined as \cref{def:geo-avg-effect-coef}, satisfies:
    \begin{equation*}
        \left| \bar {\mathcal {E}}_{\phi_1}(\vx) \right|^n \geq \exp(\int_0^1\nabla\cdot \vv(\vx_t,t)dt)=\frac{p_0(\vx)}{p_1(\phi_1(\vx))},
    \end{equation*}
\end{lemma}

\begin{proof}
    For any $t\in[0,1]$ and $\vx_0$, define $J(t)=\frac{\partial \vx_t}{\partial \vx_0}$, where $\vx_t$ is defined in \cref{lemmaode}. Then $J(0)=I$, $J(1)=J_{\phi_1}(\vx_0)$, and
    \begin{align}
        \frac{\dd}{\dd t}J(t)&=\frac{\dd}{\dd t}\frac{\partial \vx_t}{\partial \vx_0}\\
        &=\frac{\partial \vv(\vx_t,t)}{\partial \vx_0}\\
        &=\frac{\partial \vv(\vx_t,t)}{\partial\vx_t}\circ J(t).
    \end{align}
Multiplying $J^{-1}(t)$ in both sides, we have
\begin{equation}\label{eq:dJ-Jinv}
    \frac{\dd J(t)}{\dd t} \circ J^{-1}(t) = \frac{\partial \vv(\vx_t,t)}{\partial\vx_t}.
\end{equation}
Denote the $i$-th eigenvalue of $J(t)$ as $\lambda_i(J(t))$. With \cref{eq:dJ-Jinv}, we can obtain $\frac{\dd}{\dd t}\log |\prod_{i=1}^n \lambda_i(J(t))|$ as follows 
\begin{align}
    \frac{\dd}{\dd t}\log |\prod_{i=1}^n \lambda_i(J(t))|
    &=\sum_{i=1}^n\frac{\dd}{\dd t}\log|\lambda_i(J(t))|\\
    &=\sum_{i=1}^n\frac{\dd \lambda_i(J(t))}{\dd t}\lambda_i(J(t))^{-1}\\
    &=\text{tr}\left(\frac{\dd J(t)}{\dd t}\circ J^{-1}(t)\right)\\
    &=\text{tr}\left(\frac{\partial \vv(\vx_t, t)}{\partial \vx_t}\right)\\
    &=\nabla\cdot \vv(\vx_t,t).
\end{align}
 Thus deduce that
\begin{equation}
    \log |\prod_{i=1}^n \lambda_i(J(t))|=\int_0^1\nabla\cdot { \vv(\vx_t,t)}\dd t.
\end{equation}
On the other hand, for the probability density we can derive that:
\begin{align}
    \frac{\dd \log p_t(\vx_t)}{\dd t}&=\nabla \log p_t(\vx_t)\cdot \frac{\dd \vx_t}{\dd t}+\frac{\partial}{\partial t}\log p_t(\vx_,t)\\
    &=\nabla \log p_t(\vx_t)\cdot \vv(\vx_t,t)+\frac{\partial}{\partial t}\log p_t(\vx_,t)\\
    &=\frac{\nabla p_t(\vx_t)}{p_t(\vx_t)}\cdot \vv(\vx_t,t)-\frac{\nabla\cdot (p_t(\vx_t)\vv(\vx_t,t))}{p_t(\vx_t)}\label{FPused}\\
    &=-\nabla \cdot \vv(\vx_t,t),
\end{align}
where \cref{FPused} is a result from Fokker-Planck equation: $\frac{\partial p_t(\vx_t)}{\partial t}=-\nabla\cdot \left(p_t(\vx_t)\vv(\vx_t,t)\right)$.
Thus we conclude that
\begin{equation}
         \log p_t(\vx_t)+\log|\prod_{i=1}^n \lambda_i(J(t))|=\log p_0(\vx_0)+\log|\prod_{i=1}^n \lambda_i(J(0))|.
         \end{equation}
Note that $J(0)$ is the identity matrix, hence we obtain
\begin{equation}
         \left| \bar {\mathcal {E}}_{\phi_1}(\vx) \right|^n\geq|\prod_{i=1}^n \lambda_i(J(t))|=\frac{p_{0}(\vx_0)}{p_t(\vx_t)},
\end{equation}
where the first inequality holds because the product of the singular values of a matrix is greater than the absolute value of the product of its eigenvalues.
\end{proof}

\begin{lemma} \label{prob-ineqlemma}
    Let $\gamma$ be the standard Gaussian probability density function and $G^{-1}$ be the inversion function. For all $\vx \in \supp(\datadist)$, and for all $K > 0$, $k > 0$, we have
    \begin{equation}
    \begin{aligned}
                \realdist\left( \left\{\vx: \frac{p_{data}(\vx)}{\gamma(G^{-1}(\vx))}<K \right\} \right)\ge& \realdist\left(\left\{\vx:{p_{data}(\vx)}<kK\right\}\right)\\
                &-\realdist\left(\left\{\vx:\gamma(G^{-1}(\vx))\leq k\right\}\right).
    \end{aligned}
    \end{equation}
\end{lemma}

\begin{proof}
\begin{align}
&\realdist\left( \left\{\vx:\frac{p_{data}(\vx)}{\gamma(G^{-1}(\vx))}<K\right\}\right)\nonumber\\
=&
\realdist\left(\{\vx:p_{data}(\vx)<K\gamma(G^{-1}(\vx))\}\cap\{\vx:\gamma(G^{-1}(\vx))>k\}\right)\\
&+
\realdist\left(\{\vx:p_{data}(\vx)<K\gamma(G^{-1}(\vx))\}\cap\{\vx:\gamma(G^{-1}(\vx))\leq k\}\right)\\
\geq& \realdist\left(\{\vx:\vx:p_{data}(\vx)<K\gamma(G^{-1}(\vx))\}\cap\{\vx:\gamma(G^{-1}(\vx))>k\}\right)\\
\ge& \realdist\left(\{\vx:\gamma(G^{-1}(\vx))>k\}\cap\{\vx:p_{data}(\vx)<kK \}\right)\\
\ge& \realdist\left(\{\vx:p_{data}(\vx)<kK \}\right)-\realdist\left(\{\vx:\gamma(G^{-1}(\vx)) \leq k\}\right).
\end{align}
\end{proof}

Now the conclusion of \cref{thm:prob-lower-bound} is a direct corollary of the two lemmas above.
For convenience, we first repeat \cref{thm:prob-lower-bound} to be proved here.
\THMproplowerbound*
\begin{proof}
        From \cref{push forward} and \cref{prob-ineqlemma}, we have
    \begin{align}
         \instableprob=&\realdist\left(\{\vx: \bar{\gE}_G(G^{-1}(\vx)) > M \}\right) \\
         \geq&\realdist\left(\left\{ \vx:\frac{\pgen(\vx)}{\gamma(G^{-1}(\vx))}< M^{-n}\right\}\right) \label{eq:thm41-proof-first-ineq}\\
         \geq& \realdist\left(\left\{\vx: \pgen(\vx) 
         < \frac{1}{(2\pi M^2)^{\frac n 2}}e^{-\frac{2n+3\sqrt{2n}}{2}}\right\}\right) \nonumber \label{eq:thm41-proof-second-ineq} \\
         &-\realdist\left(\left\{\vx:\gamma(G^{-1}(\vx))<\frac{1}{(2\pi )^{\frac n 2}}e^{-\frac{2n+3\sqrt{2n}}{2}}\right\}\right),
    \end{align}
    where the first inequality \cref{eq:thm41-proof-first-ineq} follows from \cref{push forward} if we take $p_0$ as the Gaussian density function $\gamma$ and $p_t$ as the density of generated distribution $\pgen$ during the generation process, the second inequality \cref{eq:thm41-proof-second-ineq} follows from \cref{prob-ineqlemma} with $k=\frac{1}{(2\pi )^{\frac n 2}}e^{-\frac{2n+3\sqrt{2n}}{2}}$.
\end{proof}

\subsection{Proof of \cref{thm:prob-lower-bound}}\label{subsec:app_proof_41}
Before the formal proof of \cref{thm:prob-lower-bound}, we first provide the definition of L\'{e}vy-Prokhorov metric~\citep{billing}.

\subsubsection{Definition of L\'{e}vy-Prokhorov metric}\label{sec:app_definition}
\begin{definition}[L\'{e}vy-Prokhorov metric]\label{def:lp-metric}
    For a subset $A\subset \sR^n$, define the $\epsilon$- neighborhood of $A$ by
    \begin{equation}
        A^{\epsilon}:=\{p\in\mathbb R^n \,:\, \exists q\in A,\; d(p,q)<\epsilon\}=\underset{p\in A}{\bigcup}B_{\epsilon}(p).
    \end{equation}
    where $B_{\epsilon}(p)$ is the open ball of radius $\epsilon$ centered at $p$. Let $\mathcal P(\mathbb R^n)$ denotes the collection of all probability measures on $\mathbb R^n$
    The L\'{e}vy-Prokhorov metric $\pi:\mathcal P(\mathbb  R^n)^2\rightarrow  [0,+\infty)$is defined as follows:
    \begin{equation}
        \pi(\mu,\nu):=\inf\{\epsilon>0\,:\,\mu(A)\leq  \nu(A^{\epsilon})+\epsilon \;and\; \nu(A)\leq  \mu(A^{\epsilon})+\epsilon\}.
    \end{equation}
    This metric is equivalent to weak convergence of measures~\citep{billing}.
\end{definition}

\subsubsection{Proof of \cref{thm:mix-of-gaussian-neighbor-approx}}
For convenience, we first repeat \cref{thm:mix-of-gaussian-neighbor-approx} to be proved here.
\THMuniversalapprox*
\begin{proof}
    $n$-dimensional Hermite functions form a complete orthogonal basis of $L^2(\mathbb R^n)$ space~\citep{stein2003fourier}. Hermite functions can be expressed as:
    \begin{equation}
        \phi_\alpha(\vx) = H_\alpha(\vx) e^{-\|\vx\|^2}.
    \end{equation}
    Here $\alpha$ is a multi-index, an ordered $n$-tuple of nonnegative integers. $H_\alpha(\vx)$ is a polynomial of $\alpha$ order called Hermite polynomials. By Plancherel theorem, the Fourier transform is a isomorphism on $L^2(\mathbb R^n)$.~\citep{folland2013real} It is sufficient to prove the Fourier transform of $\{p(\vx) = \sum_{i=1}^m a_i f_i * g_{w_i}(\vx)\}$ can converge to Hermite functions in $L^2$. We denote the Fourier transform of any function $h$ as $\widehat{h}$, with the wide hat notation $\widehat{\cdot}$. Thus
    \begin{equation}
        \sum_{i=1}^m \widehat{a_i f_i * g_{w_i}}(\vx)=\sum_{i=1}^m a_i\hat{f_i}\cdot e^{-\frac 1 2w_i^2\|\vx\|^2}.
    \end{equation}
    Since the continuous compact support function $C_c(\mathbb R^n)$ is dense in $L^2(\mathbb R^n)$~\citep{folland2013real}, by the isomorphism, $\{a_i\hat{f_i}\}$ is also dense in $L^2(\mathbb R^n)$. $H_\alpha(\vx)e^{-\frac 1 2 \|\vx\|^2}$ is a $L^2$ function. 
 Then we derive that $\forall \phi_\alpha,\forall\delta>0,\exists a f$ and take $w_i=1$,
    \begin{align}
        \|  \widehat{a f\ast g_{1}}(\vx)-\phi_\alpha(\vx)\|_{L^2}&=\|(H_\alpha(\vx)e^{-\frac 1 2 \|\vx\|^2}-a\hat f)e^{-\frac 1 2 \|\vx\|^2}\|_{L^2}\\
        &\leq\|(H_\alpha(w_i\vx)e^{-\frac 1 2 \|\vx\|^2}-a\hat f)\|_{L^2}\|e^{-\frac 1 2 \|\vx\|^2}\|_{L^\infty}\\
        &\leq \delta.
    \end{align}
    The first inequality is Hölder's inequality, and the second inequality follows from the density of $af$ in the $L^2$ space.
    $\forall \epsilon >0, \,h(x)\in L^2(\mathbb R^n)$, there exists a finite set of multi-index $\{\alpha_k,k\leq N\}$ and sequence $\{\theta_{\alpha_k}\}$ such that
    \begin{equation}
        \|\sum_{k\leq N}\theta_{\alpha_k}\phi_{\alpha_k}(\vx)-\hat h(\vx)\|_{L^2}\leq \frac{\epsilon}{2},
    \end{equation}
    then for each $\phi_{\alpha_k}$ choose a $a_{\alpha_k} f_{{\alpha_k}} \ast g_{w_{\alpha_k}}$ such that $\|a_{\alpha_k} f_{{\alpha_k}} \ast g_{w_{\alpha_k}}(\vx)-\theta_{\alpha_k}\phi_{\alpha_k}(\vx)\|_{L^2}\leq \frac{\epsilon}{2N}$, we obtain that
    \begin{align}
        &\|\sum_{k\leq N}a_{\alpha_k} f_{{\alpha_k}} \ast g_{w_{\alpha_k}}(\vx)-\hat h(\vx)\|_{L^2}\\
        \leq &\|\sum_{k\leq N}a_{\alpha_k} f_{{\alpha_k}} \ast g_{w_{\alpha_k}}(\vx)-\sum_{k\leq N}\theta_{\alpha_k}\phi_{\alpha_k}(\vx)\|_{L^2}+ \|\sum_{k\leq N}\theta_{\alpha_k}\phi_{\alpha_k}(\vx)-\hat h(\vx)\|_{L^2}\\
        \leq &\epsilon,
    \end{align}
    which shows the density of $\sum_{k\leq N}a_{i} f_{{i}} \ast g_{w_{i}}$ in $L^2(\mathbb R)$. 
    Now we prove the density in L\'{e}vy-Prokhorov metric, $\forall f$ is a density function of $\mathbb R^n$, $\forall \epsilon>0$, we can construct an $\epsilon$-decomposition as follow:
    Since $f$ is integrable, there is a compact set $K_f$ in $\mathbb R$ with $\int_{K_f^c}f(x)\dd \vx<\epsilon$ and $\int_{\partial K_f}f(\vx)\dd \vx=0$. The interior of K is an open set $O$ and by the structure theorem of open sets, there exists a decomposition $O=\bigcup_i O_i$, $\{O_i\}$ are disjoint open sets and to make the index set finite, the distance between $O_i$ is no less than some $d>0$. Then we have $O = \bigcup_{i\leq N'}O_i$

    Then we define $a_i\defeq\int_{O_i}f(\vx)\dd \vx$ and $f_i(\vx)\defeq \frac{1}{a_i}f(\vx)I_{O_i}$, $\forall A'\subset \bigcup_{i\leq N'}O_i$, denote $A'\cap O_j$ as $A'_j$
    \begin{align}
        &\int_{A'}|f-\sum_{i}a_if_i\ast g_{w_i}|\dd \vx \\
        =&\sum_{j}\int_{A_j'}|f-\sum_{i}a_if_i\ast g_{w_i}|\dd \vx\\
        \leq& \sum_{j}\int_{A_j'}|f-a_jf_j\ast g_{w_j}|\dd \vx+\sum_{j}\int_{A_j'}|a_jf_j\ast g_{w_j}-\sum_{i}a_if_i\ast g_{w_i}|\dd \vx.
    \end{align}
    $g_{w_j}$ is an approximate identity about $w_j$ and when $w_j$ is sufficiently small, $\int_{A_j'}|f-a_jf_j\ast g_{w_j}|\dd \vx\leq \|f-a_jf_j\ast g_{w_j}\|_{L^1}\leq \frac \epsilon {2N'}$ and
    \begin{equation}
          \int_{A_i'}|a_jf_j\ast g_{w_j}-\sum_{i}a_if_i\ast g_{w_i}|\dd \vx= \sum_{i\neq j}\int_{A_i'} |a_if_i\ast g_{w_i}|\dd\vx\leq\sum_{i\neq j}\frac{1}{2\pi w_i}\exp(-\frac{d^2}{2w_i^2}),
    \end{equation}
    when $w_i$ is sufficiently small, $\sum_{i\neq j}\frac{1}{2\pi w_i}\exp(-\frac{d^2}{2w_i^2})\leq  \frac{\epsilon}{2N'}$. From this we conclude that $\forall A'\subset \bigcup_{i\leq N'}O_i$, we have
    \begin{align}
        \int_{A'}f(\vx)\dd \vx \leq \int_{A'}\sum_{i}a_if_i\ast g_{w_i}(\vx)\dd \vx +\epsilon,\\
        \int_{A'}\sum_{i}a_if_i\ast g_{w_i}(\vx)\dd \vx \leq \int_{A'}f(\vx)\dd \vx +\epsilon.
    \end{align}
    And $\forall A'\subset O^c$, $\int_{A'}f(\vx)\dd\vx\leq \epsilon$, then it is sufficient to prove
    \begin{equation}
        \int_{A'}\sum_{i}a_if_i\ast g_{w_i}(\vx)\dd\vx\leq \int_{A'^\epsilon}f(\vx)\dd\vx +\epsilon
    \end{equation}
    and this could also be obtained by the $L^1$ convergence property of approximate identity: $\int_{A'}\sum_{i}a_if_i\ast g_{w_i}(\vx)\dd\vx < \epsilon$ with appropriate $w_i$.

    In summary, $\forall A'\subset\mathbb R^n$, 
    \begin{align}
        \int_{A'}f(\vx)\dd\vx
        &= \int_{A'\cap O}f(\vx)\dd\vx+ \int_{A'\cap O^c}f(\vx)\dd\vx\\
        &\leq \int_{(A'\cap O)^\epsilon}\sum_{i}a_if_i\ast g_{w_i}(\vx)\dd\vx +\epsilon+\epsilon\\
        &\leq \int_{(A')^{2\epsilon}}\sum_{i}a_if_i\ast g_{w_i}(\vx)\dd\vx +2\epsilon,\\
        \int_{A'}\sum_{i}a_if_i\ast g_{w_i}(\vx)\dd\vx
        &= \int_{A'\cap O}\sum_{i}a_if_i\ast g_{w_i}(\vx)\dd\vx+ \int_{A'\cap O^c}\sum_{i}a_if_i\ast g_{w_i}(\vx)\dd\vx\\
        &\leq \int_{(A'\cap O)^\epsilon}f(\vx)\dd\vx +\epsilon+\epsilon\\
        &\leq \int_{(A')^{2\epsilon}}f(\vx)\dd\vx +2\epsilon.
    \end{align}
    By the arbitrariness of $\epsilon$, we conclude that the set of distributions with probability density function $\{\sum_{i}a_if_i\ast g_{w_i}\}$ is dense.
\end{proof}

\subsection{Proof of \cref{thm:main}}
For convenience, we first repeat \cref{thm:main} to be proved here.
\THMmain*
\begin{proof}
The overall idea of the proof is to find a value $M$ related to $w_i$ such that as $n \to \infty$, we have $w_i \to 0$, thereby implying that both $\epsilon$ and $\delta$ tend to 0 and $M\rightarrow\infty$.
From \cref{assump:gen-dist-mixture}
    \begin{align}
        p_{gen}(\vx)&=\sum_{i=1}^m\int_{\mathbb R^n}a_if_i(\vy)g_{w_i}(\vx-\vy)\dd \vy\\
        &=\sum_{i=1}^m\int_{\underset{i\leq m}{\bigcup}O_i}a_if_i(\vy)\frac{1}{(2\pi w_i^2)^{\frac n 2}}e^{-\frac{\|\vx-\vy\|^2}{2w_i^2}}\dd \vy\\
        &\leq \sum_{i=1}^m\int_{\underset{i\leq m}{\bigcup}O_i}a_if_i(y)\frac{1}{(2\pi w_i^2)^{\frac n 2}}e^{-\frac{\min_{i\leq m}d(\vx, O_i)^2}{2w_i^2}}\dd \vy\\
        &=\sum_{i=1}^m a_i\frac{1}{(2\pi w_i^2)^{\frac n 2}}e^{-\frac{\min_{i\leq m}d(\vx, O_i)^2}{2w_i^2}}\\
        &\leq \max_{i\leq m}\left\{\frac{1}{(2\pi w_i^2)^{\frac n 2}}e^{-\frac{\min_{i\leq m}d(\vx, O_i)^2}{2w_i^2}}\right\}.
    \end{align}
    We define a dominating function
    \begin{equation}\label{def:h}
        h_{w_1,...,w_m}(\vx)\defeq \max_{i\leq m}\left\{\frac{1}{(2\pi w_i^2)^{\frac n 2}}e^{-\frac{\min_{i\leq m}d(\vx, O_i)^2}{2w_i^2}}\right\},
    \end{equation} which is monotonically decreasing about $\min_{i\leq m}d(\vx, O_i)$ and can be written as $h_{w_1,...,w_m}(x)=h_{w_1,...,w_m}'(\min_{i\leq m}d(\vx, O_i))$, where $h_{w_1,...,w_m}'(x)\defeq\max_{i\leq m}\left\{\frac{1}{(2\pi w_i^2)^{\frac n 2}}e^{-\frac{x^2}{2w_i^2}}\right\} $.  From the above inequality we obtain:
    \begin{align}
        &\left\{\vx:\pgen (\vx)\geq \frac{1}{(2\pi M^2)^{\frac n 2}}e^{-\frac{2n+3\sqrt{2n}}{2}}\right\}\\\subset&\left\{\vx: h_{w_i,...,w_m}(\vx)\geq \frac{1}{(2\pi M^2)^{\frac n 2}}e^{-\frac{2n+3\sqrt{2n}}{2}}\right\}\\
        =&\left\{\vx: \min_{i\leq m}d(\vx, O_i)\leq r_M\right\},\label{76}
    \end{align}
where
\begin{equation}\label{eq:def-rm}
    r_M \defeq \|h_{w_i,...,w_m}'^{-1}\left(\frac{1}{(2\pi M^2)^{\frac n 2}}e^{-\frac{2n+3\sqrt{2n}}{2}}\right)\|.
\end{equation}
Based on this, taking $M\leq M_0\defeq \min_{1\leq i\leq m}\exp \left(\frac{1}{8}\frac{\bar{d}_{\min}^2}{w_i}-\ln \frac{1}{w_i} + 2 + 3\sqrt{\frac{2}{n}} \right)$, we observe that $M_0\rightarrow\infty$ as $w_i\rightarrow 0 $, and in the following part we will proof $w_i\rightarrow 0 $ as $n\rightarrow\infty$. Here $\bar{d}_{\min}$ is the minimum of distance between $B_i$.
    
    Since $h_{w_1,...,w_m}'$ is radial, once $M$ is fixed, $r_M$ is well defined. Thus there exists a unique $w_j$ in \cref{def:h} such that for any $\vx$ satisfying $\min_{i\leq m}d(\vx, O_i)=r_M$, we have
    \begin{equation}
        h_{w_1,...,w_m}(\vx)=\frac{1}{(2\pi w_j^2)^{\frac n 2}}e^{-\frac{\min_{i\leq m}d(\vx, O_i)^2}{2w_j^2}}.
    \end{equation}
    Then from \cref{76} we deduce that
    \begin{align}
         &\left\{\vx:\pgen (\vx)\geq \frac{1}{(2\pi M^2)^{\frac n 2}}e^{-\frac{2n+3\sqrt{2n}}{2}}\right\}\\\subset&\left\{\min_{i\leq m}d(\vx, O_i)\leq \frac{1}{2}\sqrt{nw_j}\bar{d}_{\min}\right\}=\bigcup_{i\leq m}\left\{d(\vx, O_i)\leq \frac{1}{2}\sqrt{nw_j}\bar{d}_{\min}\right\},
    \end{align}
    where $\bar{d}_{\min}$ is the dimension-normalized $d_\min$, as defined in \cref{assump:sparsity-group}.
    As we discussed previously in the third of \cref{assump:sparsity-group}, all $O_i$ is contained by a cube $B_i$, consequently $\left\{d(\vx, O_i)\leq \frac{1}{2}\sqrt{nw_j}\bar{d}_{\min}\right\}\subset \left\{d(\vx, B_i)\leq \frac{1}{2}\sqrt{nw_j}\bar{d}_{\min}\right\}$ and to calculate the Lebesgue measure of the geometric body obtained by expanding an $n$-dimensional cube $B_i$ outward by a distance $\frac{1}{2}\sqrt{nw_j}\bar{d}_{\min}$, we can utilize the Steiner formula~\citep{schneider1993convex}, which expresses the volume as the sum of the expansion volumes of different-dimensional faces of the original cube, with each expanded volume contribution being the product of the corresponding ball volume.

    For each $k\leq n$, the number of $k$-dimensional faces is $\binom{n}{k}\cdot 2^{n-k}$ and each has a expansion volume contribution $b_i^k\cdot \omega_{n-k}(\frac{1}{2}\sqrt{nw_j}\bar{d}_{\min})^{n-k}(\frac{1}{2})^{n-k}$. Based on this, we conclude that
    \begin{equation}
        m\left\{d(\vx, B_i)\leq \frac{1}{2}\sqrt{nw_j}\bar{d}_{\min}\right\}=\sum_{k=0}^{n}\binom{n}{k}b_i^k\cdot\omega_{n-k}(\frac{1}{2}\sqrt{nw_j}\bar{d}_{\min})^{n-k}.
    \end{equation}
    Comparing with the support of $\realdist$, which contain some disjoint cubes centered at $x_i$ and the edge length is no less than ${\sqrt{3\pi e}}b_i$, we can make the following analysis:
    \begin{align}
        \frac{ m\left\{d(\vx, O_i)\leq \frac{1}{2}\sqrt{nw_j}\bar{d}_{\min}\right\}}{({\sqrt{3\pi e}}b_i)^n}&\leq \frac{\pi^{\frac{n}{2}}(b_i\sqrt{n}+\frac{1}{2}\sqrt{nw_j}\bar{d}_{\min})^n}{\Gamma(\frac{n}{2}+1)({\sqrt{3\pi e}}b_i)^n}\\
        &\sim\frac{\pi^{\frac{n}{2}}}{\sqrt{\pi n}(\frac{n}{2e})^{\frac{n}{2}}}\frac{(b_i\sqrt{n}+\frac{1}{2}\sqrt{nw_j}\bar{d}_{\min})^n}{({\sqrt{3\pi e}}b_i)^n}\quad(\textit{Stirling formula}).\label{82}
    \end{align}
    With the low probability region assumption in \cref{assump:sparsity-group}, we have
    \begin{align}
        &\int_{O_i}f_i\ast g_{w_i}(\vx)\dd \vx\\
        \leq&\int_{B_i}f_i\ast g_{w_i}(\vx)\dd \vx\\
        =&\int_{B_i} \int f_i(\vy)g_{w_i}(\vx-\vy)\dd \vy\dd \vx\\
        =&\int f_i(\vy)\int_{B_i} g_{w_i}(\vx-\vy)\dd \vx\dd \vy\\
        \leq&\int f_i(\vy)\int_{B_i} g_{w_i}(\vx-\vx_i)\dd \vx\dd \vy\\
        = &\int_{B_i} g_{w_i}(\vx-\vx_i)\dd \vx\\
        \leq &\left(\int_{-b_i}^{b_i}\frac{1}{2\pi w_i^2}e^{-\frac{x^2}{2w_i^2}}dx\right)^n.
    \end{align}
    As mentioned previously we obtain that $\int_{O_i}f_i\ast g_{w_i}(\vx)\dd \vx>\alpha_i$ leads to $w_i$ converging to zero when $n\rightarrow\infty$, which means $M_0\rightarrow\infty$.
    Substituting this result into \cref{82}, we obtain: $\frac{ m\left\{d(\vx, O_i)\leq \frac{1}{2}\sqrt{nw_j}\bar{d}_{\min}\right\}}{({\sqrt{3\pi e}}b_i)^n}$ converge to 0 as $n\rightarrow\infty$. Thus we can deduce that
    \begin{equation}\label{measure-converge}
        \lim_{n\rightarrow\infty}\frac{m\left\{\bigcup_{i\leq m}\left\{d(\vx, O_i)\leq \frac{1}{2}\sqrt{nw_j}\bar{d}_{\min}\right\}\right\}}{m\left\{\supp (\realdist)\right\}}=0.
    \end{equation}
    
    See $\realdist$ on $\mathbb R^n$ as taking finite pixels from a continuous function $f\in C(S)$ representing a infinitely precise real image. $S$ is a compact subset of $\mathbb R^2$. The pixels form a finite $\epsilon$-dense set $S_n$ of the separable space $S$. And $S_\infty=\bigcup_{n=1}^\infty S_n$ is a countable dense set of S. And all the continuous function on $S_\infty$ is $C(S)$. $\bigcup_{n}\mathcal B(\mathbb R^n)$ forms a semialgebra with a premeasure that generates $\mathcal B(C(S))$. With the Kolmogorov extension theorem~\citep{tao2021introduction}, there is a unique distribution $\realdist'$ on $\mathcal B(C(S))$ generated by all $\realdist^{(n)}$ on $\mathbb R^n$. Additionally, considering that step functions could uniformly converge to any $f\in C(S)$. The finite n-pixel image space can also be seen as step functions on $S$: $S = \bigcup_{i\leq n}U_i$ is a decomposition of disjoint set, each $U_i$ represent a pixel and $m(U_i)=\frac{m(S)}{n}$. For any $C(S)$ value random variable X, there is a step function value random variable $X_n = \sum_{i\leq n}\vx_iI_{U_i}$ as $X$'s n-dimensional projection and $X_n$ converges to $X$ almost surely in the function space equipped with the uniform norm:
    \begin{equation}
        P(\omega:\|X_n(\omega)-X(\omega)\|\rightarrow0)=1.
    \end{equation}
    From this, we derive the convergence of probability law :
    \begin{equation}
        P^{X_n}\overset{w}{\longrightarrow}P^{X},
    \end{equation}
    which means $\realdist^{(n)}$ converge to $\realdist'$ in the Banach space of continuous and step functions on S with uniform norm~\citep{ikeda2014stochastic}. 

    Since $\realdist^{(n)}$ is absolutely continuous about $\pi_{gen}$ till the infinite-dimensional case, $\delta'=\pi_{gen}\left(G(\vz):\|z\|^2\geq 2n+3\sqrt{2n}\right)=1-F(2n+3\sqrt{2n}; n)$, where $F(\cdot;n)$ is the cumulative distribution function of chi-square distribution $\chi_{n}^2$ and $1-F(2n+3\sqrt{2n}; n)$ converges to 0. And the set $B_n=\left\{\right\|z\|^2\geq 2n+3\sqrt{2n}\}$ when placed in the function space means $\|f_z\|_{L^2}\geq m(S)(2+3\sqrt{\frac{2}{n}})$, $f_z$ is the step function associated with $z$, as a result we obtain $B_{n}\subset B_{n+1}$, so $B_{\infty} = \lim_{n\rightarrow\infty} B_n = \bigcup_{n=1}^{\infty}B_n$ and $\lim_{n\rightarrow\infty} \realdist^{(n)}(B_n)=\realdist'(B_{\infty})$. While $\pi_{gen}'(B_{\infty})=\lim_{n\rightarrow\infty}\pi_{gen}\left(G(\vz):\|z\|^2\geq 2n+3\sqrt{2n}\right)=0$, we derive that $\lim_{n\rightarrow\infty}\delta=\lim_{n\rightarrow\infty}\realdist\left(G(\vz):\|z\|^2\geq 2n+3\sqrt{2n}\right)=\lim_{n\rightarrow\infty} \realdist^{(n)}(B_n)=0$.

    Using the same analytical approach, we can obtain the ratio of maximum and minimum $p_{real}$ on $\bigcup_{i\leq m}\left\{d(\vx, O_i)\leq \frac{1}{2}\sqrt{nw_j}\bar{d}_{\min}\right\}$: $\frac{C}{C_0}$ has some limit value. Combining with \cref{measure-converge}, we arrive at the conclusion that
    \begin{align}
        \lim_{n\rightarrow\infty}\epsilon
        =&\lim_{n\rightarrow\infty}\realdist\left(\bigcup_{i\leq m}\left\{\vx:d(\vx, O_i)\leq \frac{1}{2}\sqrt{nw_j}\bar{d}_{\min}\right\}\right)\\=&\lim_{n\rightarrow\infty}\frac{\realdist\left(\bigcup_{i\leq m}\left\{\vx:d(\vx, O_i)\leq \frac{1}{2}\sqrt{nw_j}\bar{d}_{\min}\right\}\right)}{\realdist(\supp (\realdist))}=0.
    \end{align}
\end{proof}

\section{Experimental settings} \label{supp:exp-setting}

\subsection{Experiments on numerical cases}\label{supp:exp-setting_1}

\paragraph{Settings for experiments in \cref{subsec:phenomenon-by-exp}} 

To verify that \phenomenon~indeed exists in diffusion generation, we conduct experiments on a two-dimensional diffusion model with a mixture of Gaussians as the generation distribution, consisting of three Gaussian components. To compute the \intrinsiceffectmetric~and obtain the results shown in \cref{fig:phenomenon-example}(b), we utilize the finite difference to estimate the \intrinsiceffectmetric. The specific steps are as follows:

\begin{enumerate}
    \item Uniformly sample initial points as on a $201 \times 201$ uniform grid of the area $[-1, 1] \times [-1, 1]$.
    Denote each initial point as $\vx[i, j]$, where $i$ denotes the index along $x$-axis, and $j$ denotes the index along $y$-axis.
    Thus, $\vx[i, j] = (-1 + \frac{i}{100}, -1 + \frac{j}{100})$ for $i,j = 0,1,\dots,200$.
    \item Numerically solve the PF-ODE in \cref{eq:pf-ode} using the RK45 solver from $t=1$ to $t=0$. Each solution at $t=1$, \textit{i.e.}, the generated sample, can be denoted as $\hat{G}(\vx[i, j])$ for each initial point $\vx[i, j]$.
    \item Estimate the \intrinsiceffectmetric~as $\gE_G(\vx[i,j], \vn_y) \approx \|\hat{G}(\vx[i,j+1]) - \hat{G}(\vx[i,j])\| / \|\vx[i,j+1] - \vx[i,j]\|$.
\end{enumerate}

When computing the \effectmetric~shown in \cref{fig:phenomenon-example}(c), we use two points with larger input difference. The results demonstrate that the obtained \effectmetric~is lower bounded by the corresponding \intrinsiceffectmetric.

\paragraph{Settings for experiments in \cref{subsec:recon-error}} 

Similarly, we adopt the two-dimensional mixture of Gaussians as the generation distribution for these experiments. To obtain the correlation between reconstruction error and \effectmetric~as shown in \cref{fig:recon-error}(b), we follow the experimental procedure as below:
\begin{enumerate}
    \item Uniformly sample initial data from $[-1, 1] \times [-1, 1]$.
    \item Compute the reconstructed samples using the diffusion reconstruction process: first obtain the inverted noise $\hat{\vz}=\hat{G}^{-1}(\vx)$ for each initial data $\vx$, and then regenerate the data as $\hat{\vx}=\hat{G}(\hat{\vz})$.
    \item Calculate the reconstruction error $\gR(\vx)$ for each sample.
    \item Estimate the \intrinsiceffectmetric~by 1) applying a small perturbation noise $\vn$ to each inverted noise $\hat{\vz}$, and then 2) regenerating the data under perturbation as $\tilde{\vx}=\hat{G}(\hat{\vz} + \vn)$, and 3) finally resulting in the estimation of \intrinsiceffectmetric~as $\gE_G(\hat{\vz}, \frac{\vn}{\|\vn\|}) \approx \|\tilde{\vx} - \hat{\vx}\| / \|\vn\|$.
    \item Statistically analyze the correlation between $\gR(\vx)$ and the \intrinsiceffectmetric.
\end{enumerate}
This procedure allows us to empirically assess the relationship between instability coefficients and reconstruction inaccuracies, thereby validating the theoretical insights discussed in \cref{sec:phenomenon}.

\subsection{Experiments on text-to-image diffusion models}\label{supp:exp-setting_2}

\paragraph{Model and inference settings.}
Models used in our experiments include Stable Diffusion 3.5 Medium / Large~\citep{esser2024scaling} and FLUX.1-dev~\citep{flux2024}, accessed via the \texttt{diffusers} library~\citep{dhariwal2021diffusion} under the PyTorch framework. The inference is conducted on NVIDIA A800 GPUs.

During inference, most of the default scheduler configuration parameters are adopted for the numerical solution of the PF ODE.
As for the ODE solver, we adopt Euler method following the default setting, except for experiments in  \cref{suppsub:second-order-recon} that the Heun ODE solver is used.
Besides, the number of inference steps is adjusted to 500 steps for both diffusion inversion and re-generation. The procedure to obtain the reconstruction error and the \intrinsiceffectmetric~is similar to that in the numerical experiments. The adopted image dataset is introduced below.

\paragraph{Datasets.}
The experiments primarily utilize the MSCOCO2014 dataset~\citep{lin2014microsoft}. For the experiments presented in \cref{fig:recon-error}, we randomly select 100 images from the validation set of MSCOCO2014 for diffusion reconstruction.
As for experiments in \cref{suppsub:second-order-recon}, we adopt the same sampled images as in \cref{fig:recon-error}.

\section{Additional reconstructed images} \label{supp:recon-images}

To visually demonstrate that images are often difficult to be accurately reconstructed, we present failure cases of reconstruction by Stable Diffusion 3.5~\citep{esser2024scaling}.
These images are selected from Kodak24 dataset~\citep{kodak} in \cref{fig:supp-reconstruction-kodak} and MS-COCO 2014~\citep{lin2014microsoft} in \cref{fig:supp-reconstruction-coco}.
For the reconstruction that includes both diffusion inversion and regeneration processes, we follow the default scheduler setting in diffusers~\citep{von-platen-etal-2022-diffusers} with null text prompt and 100 inference steps for both inversion and regeneration.

\begin{figure*}[th]
\vskip 0.2in
\begin{center}
\begin{overpic}[width=\columnwidth]{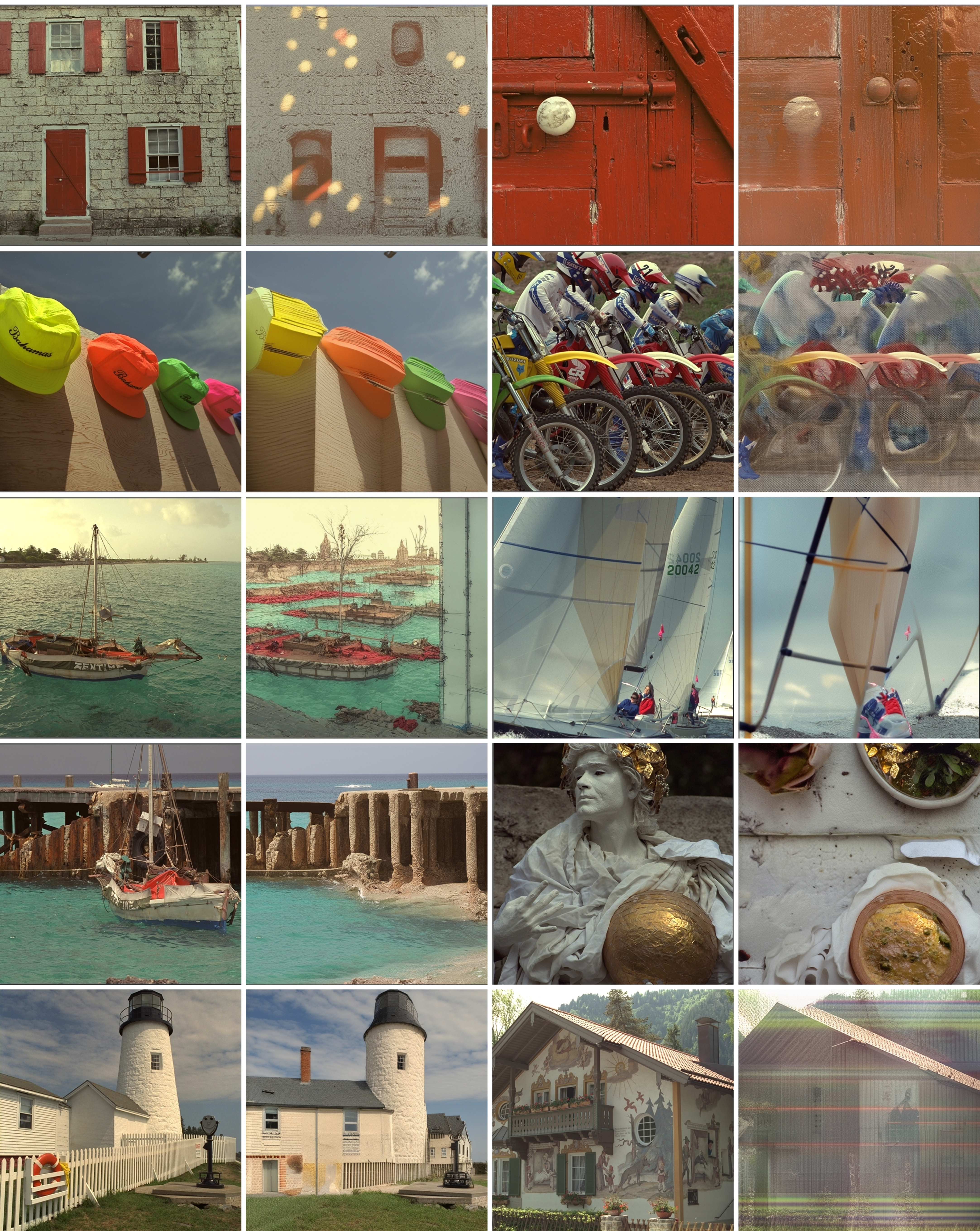}
\end{overpic}
\caption{More failure cases of reconstruction on real images from Kodak24 dataset~\citep{kodak} by Stable Diffusion 3.5~\citep{esser2024scaling}. In each row, the first and the third images are original real images, another two images are reconstructed ones.}
\label{fig:supp-reconstruction-kodak}
\end{center}
\end{figure*}

\begin{figure*}[th]
\vskip 0.2in
\begin{center}
\begin{overpic}[width=\columnwidth]{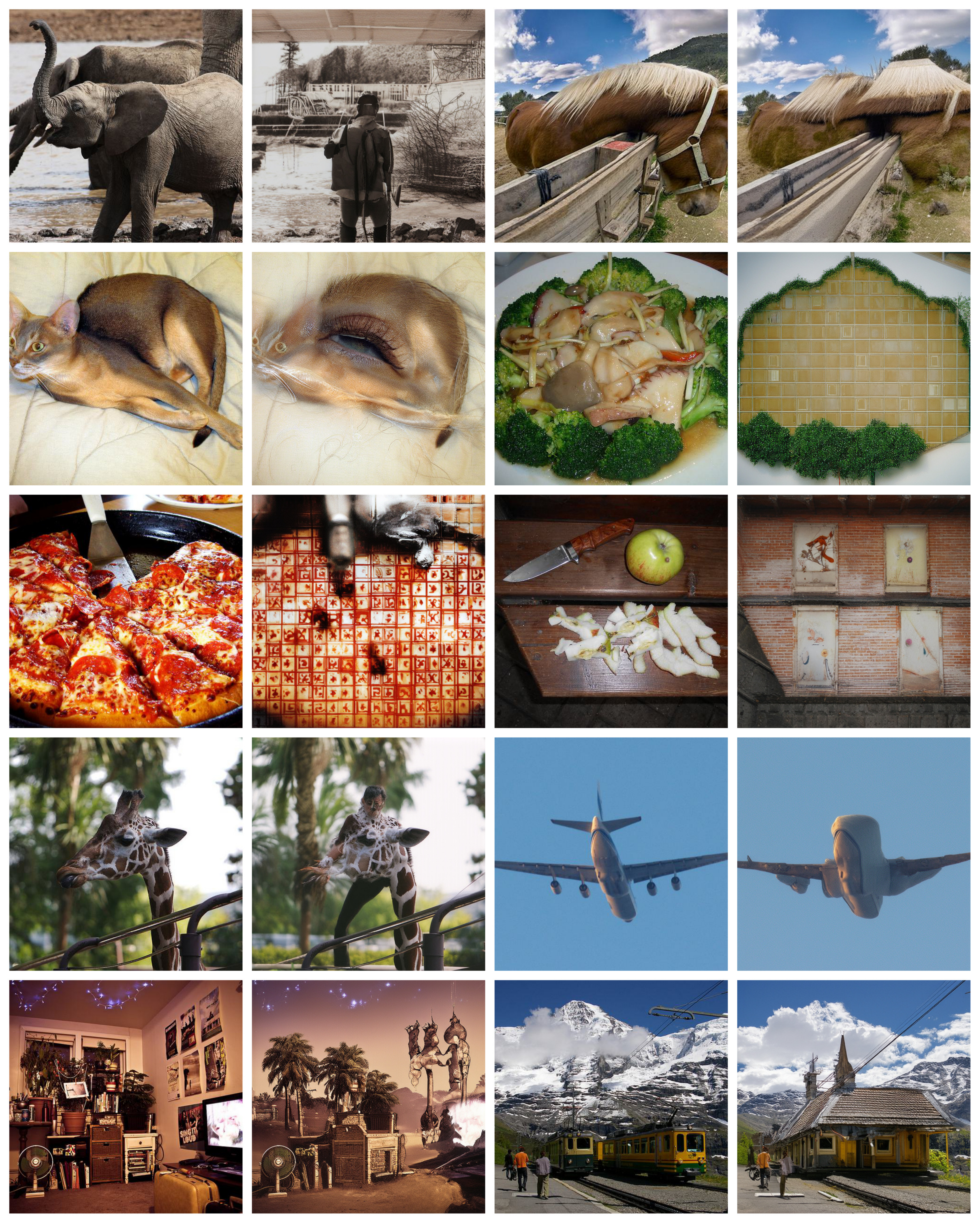}
\end{overpic}
\caption{More failure cases of reconstruction on real images from MS-COCO 2014 datadset~\citep{lin2014microsoft} by Stable Diffusion 3.5~\citep{esser2024scaling}. In each row, the first and the third images are original real images, another two images are reconstructed ones.}
\label{fig:supp-reconstruction-coco}
\end{center}
\end{figure*}

\section{Licenses of used datasets}

We list all the licenses of used datasets, code and models in \cref{tab:licenses}.

\begin{table}[hb]
\caption{Licenses of datasets, code and models used in the paper.}
\label{tab:licenses}
\begin{center}
\setlength{\tabcolsep}{3pt}
\small
\begin{tabular}{ll}
    \toprule
    Name & License \\
    \midrule
    \multicolumn{2}{l}{\textbf{Datasets}} \\
    \midrule
    MS-COCO2014~\citep{lin2014microsoft}  & \href{https://cocodataset.org/#termsofuse}{Creative Commons Attribution 4.0 License} \\
    Kodak24~\citep{kodak} & Free \\
    \midrule
    \multicolumn{2}{l}{\textbf{Code}} \\
    \midrule
    Diffusers~\citep{von-platen-etal-2022-diffusers} & \href{https://github.com/huggingface/diffusers/blob/main/LICENSE}{Apache License 2.0} \\
    \midrule
    \multicolumn{2}{l}{\textbf{Models}} \\
    \midrule
    Stable Diffusion 3.5 Medium~\citep{esser2024scaling} & \href{https://huggingface.co/stabilityai/stable-diffusion-3.5-medium/blob/main/LICENSE.md}{Stability AI Community License} \\
    Stable Diffusion 3.5 Large~\citep{esser2024scaling} & \href{https://huggingface.co/stabilityai/stable-diffusion-3.5-large/blob/main/LICENSE.md}{Stability AI Community License} \\
    FLUX.1-Dev~\citep{flux2024} & \href{https://huggingface.co/black-forest-labs/FLUX.1-dev/blob/main/LICENSE.md}{FLUX.1 [dev] Non-Commercial License} \\
    \bottomrule
\end{tabular}
\end{center}
\end{table}


\end{document}